\documentclass[a4paper]{article}
\usepackage{amsgen,amsthm,amsmath,amstext,amsbsy,amsopn,amssymb,latexsym}
\usepackage[usenames]{color}
\usepackage{array}
\usepackage{multirow}
\usepackage{graphicx}
\usepackage{amssymb}
\usepackage{amsmath}
\usepackage[pdfencoding=auto, psdextra]{hyperref}
\usepackage{bookmark}
\usepackage{listings}
\usepackage{indentfirst}
\usepackage{wrapfig}
\usepackage{subcaption}
\usepackage{float}
\usepackage[numbers,sort]{natbib}
\usepackage{bbm}
\usepackage{tikz}
\usetikzlibrary{arrows.meta, positioning}
\usepackage{booktabs}
\usepackage{appendix}
\usepackage[linesnumbered,ruled,vlined]{algorithm2e}

\usepackage{geometry}
\theoremstyle{definition}

\newtheorem{Proposition}{Proposition}

\newtheorem{Definition}{Definition}
\newtheorem{Theorem}{Theorem}
\newtheorem{example}{Example}

\DeclareMathOperator*{\argmin}{arg\,min}

\newcommand{\R}{\mathbb{R}}

\newcommand{\Var}{{\rm Var}}

\def\EE{\mathbb{E}}

\definecolor{purple}{rgb}{0.84, 0.17, 0.89}
\definecolor{red2}{rgb}{0.7, 0, 0.1}
\definecolor{blue2}{rgb}{0.1,0.1,0.65}
\usepackage{color}

\usepackage{graphicx} 
\title{Robust Sampling for Active Statistical Inference}
\author{%
 Puheng Li\textsuperscript{$\dagger$}
  \qquad
  Tijana Zrnic\textsuperscript{$\dagger$, $\diamond$}
  \qquad
  Emmanuel J. Cand\`es\textsuperscript{$\dagger$, $\triangle$}
  \\[2ex] 
  \textsuperscript{$\dagger$}Department of Statistics \\
  \textsuperscript{$\Diamond$}Stanford Data Science \\
\textsuperscript{$\triangle$}Department of Mathematics
  \\[2ex] 
  Stanford University
}
\date{}

\begin{document}

\maketitle
\begin{abstract}
Active statistical inference \citep{zrnic2024active} is a new method for inference with AI-assisted data collection. Given a budget on the number of labeled data points that can be collected and assuming access to an AI predictive model, the basic idea is to improve estimation accuracy by prioritizing the collection of labels where the model is most uncertain. The drawback, however, is that inaccurate uncertainty estimates can make active sampling produce highly noisy results, potentially worse than those from naive uniform sampling.
In this work, we present robust sampling strategies for active statistical inference. Robust sampling ensures that the resulting estimator is never worse than the estimator using uniform sampling. Furthermore, with reliable uncertainty estimates, the estimator usually outperforms standard active inference. This is achieved by optimally interpolating between uniform and active sampling, depending on the quality of the uncertainty scores, and by using ideas from robust optimization. We demonstrate the utility of the method on a series of real datasets from computational social science and survey research.
\end{abstract}

\section{Introduction}

Collecting high-quality labeled data remains a challenge in data-driven research, especially when each label is costly and time-consuming to obtain. In response, many fields have embraced machine learning as a practical solution for predicting unobserved labels, such as annotating satellite imagery in remote sensing \citep{xie2016transfer} and predicting protein structures in proteomics \citep{jumper2021highly}. Prediction-powered inference \citep{angelopoulos2023prediction} is a methodological framework showing how to perform valid statistical inference despite the inherent biases in such predicted labels.

Active statistical inference \citep{zrnic2024active} was recently introduced to further enhance inference by actively selecting which data points to label. The basic idea is to compute the model's uncertainty scores for all data points and prioritize collecting those labels for which the predictive model is most uncertain. When the uncertainty scores appropriately reflect the model's errors, \citet{zrnic2024active} show that active inference can significantly outperform prediction-powered inference (which can essentially be thought of as active inference with naive uniform sampling), meaning it results in more accurate estimates and narrower confidence intervals. However, when uncertainty scores are of poor quality, active inference can result in overly noisy estimates and large confidence intervals. This is an important limitation, seeing that there is widespread recognition that measuring model uncertainty is challenging. Large language models, for example, are often overconfident in their answers \cite{chhikara2025mind, zhang2024calibrating,zhou2023navigating}. Miscalibrated uncertainty scores also arise when there is a distribution shift between the training data and the test domain.

To illustrate the issue empirically, consider the problem of estimating the approval rate of a presidential candidate: $\theta^* = \EE[Y]$, where $Y\in\{0,1\}$ is the binary indicator of approval, using Pew post-election survey data \cite{atp79}. Here, we have demographic covariates $X_1,\dots,X_n$ corresponding to $n$ people, but we do not observe the approval indicator $Y_i$ for everyone. Rather, we have a budget $n_b < n$ on how many people we can survey and collect their $Y_i$. In addition, we have a machine learning model $f$ that we can use to obtain a cheap prediction $f(X_i)$ of $Y_i$ from the available covariates. Active inference suggests surveying those individuals where $f$ is uncertain. For example, if $f(X_i)$ is obtained by thresholding a continuous score $p(X_i) \in [0,1]$ representing the probability the model assigns to the missing label taking on the value $1$, this could mean prioritizing the collection of labels where $p(X_i)$ is close to $0.5$. 
In Figure \ref{fig:teaser}, we show the effective sample size and coverage of prediction-powered inference (uniform sampling), standard active inference, and our robust active inference method, for varying values of the budget $n_b$. The effective sample size is formally defined in Section \ref{sec:exp}; it is the number of samples the method that samples uniformly at random would need to use to achieve the accuracy of the labeling method under study.  
To demonstrate a challenge for active inference, we train $f$ on a small dataset, resulting in poorly estimated uncertainties. We see that active sampling results in a smaller effective sample size and a much larger standard deviation than simple uniform sampling. This is because the variance of the active sampling strategy is large, which is due to some extreme values of sampling probability. Meanwhile, the robust method outperforms both baselines. This is achieved by estimating the quality of the uncertainty scores and optimally interpolating between uniform and active sampling. All three methods come with provable validity guarantees, as confirmed by the achieved target coverage of $90\%$.

\begin{figure}[t]
    \centering
    \includegraphics[width=0.9\linewidth]{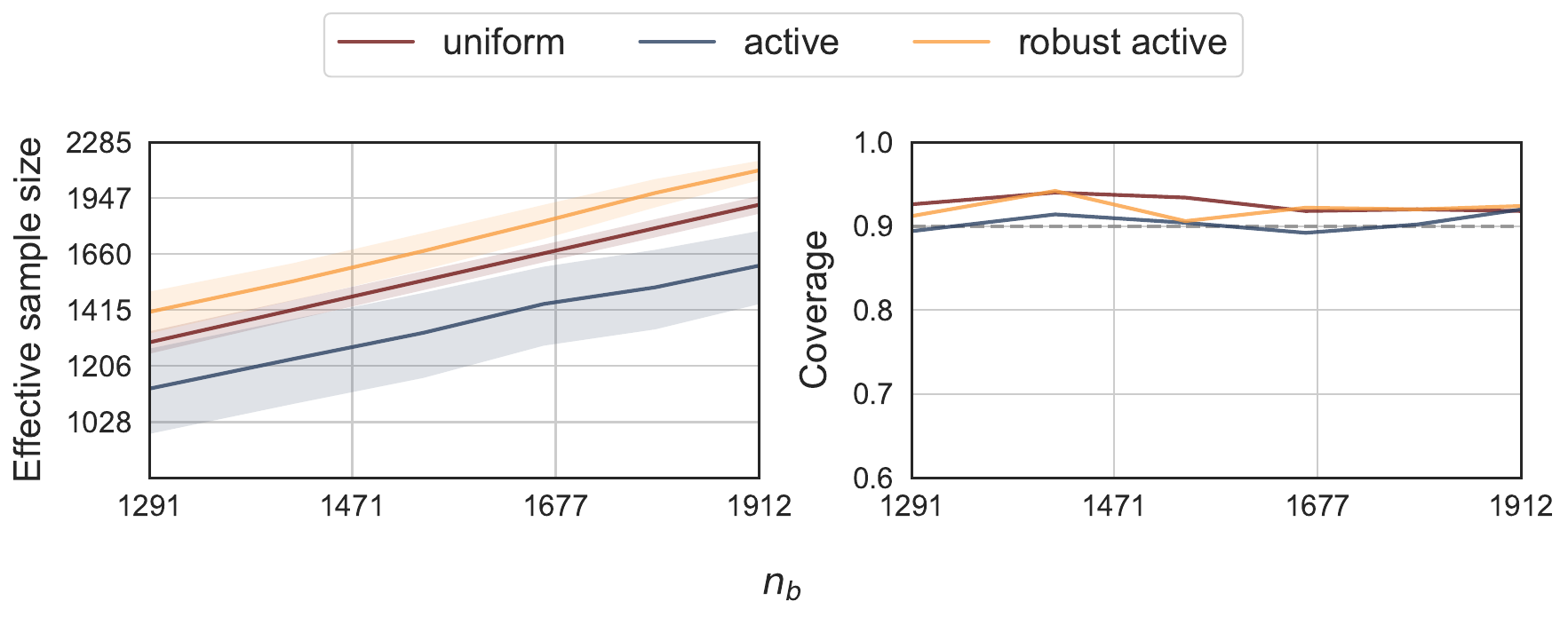}
        \caption{\textbf{Effective sample size and coverage on Pew post-election survey data.} We compare uniform, active, and robust active sampling, for different values of the sampling budget $n_b$. The target of inference is the approval rate of a presidential candidate. We show the mean and one standard deviation of the effective sample size estimated over $500$ trials; in each trial we independently sample the observed labels.}
    \label{fig:teaser}
\end{figure}

The source code for all experiments is available at:\\ { \url{https://github.com/lphLeo/Robust-Active-Statistical-Inference}}.

\subsection{Related work}
Our paper builds on active statistical inference \citep{zrnic2024active}, which itself builds on prediction-powered inference \citep{angelopoulos2023prediction} and, more generally, statistical inference assisted by predictive models \cite{mccaw2023leveraging, wang2020methods, song2024general}. There is a growing literature in this space, aimed at ensuring robustness against poor predictions~\cite{angelopoulos2023ppi++, miao2023assumption, gan2024prediction, gronsbell2024another, miao2024task, ji2025predictions}, sample efficiency when there is no good pre-trained model $f$~\cite{zrnic2024cross}, simplicity and applicability to more general estimation problems \cite{zrnic2024note, kluger2025prediction}, and handling missing covariates \cite{miao2023assumption,kluger2025prediction}. Notably, several works study adaptive label collection strategies \cite{ao2024prediction, fisch2024stratified, gligoric2024can}.

Zooming out further, at a technical level this line of work relates to semiparametric inference, missing data, and causality \cite{tsiatis2006semiparametric, robins1994estimation, robins1995semiparametric, rubin2018multiple}. In particular, the prediction-powered and active inference estimators closely resemble the augmented inverse probability weighting (AIPW) estimator~\cite{robins1994estimation}.

Our work also connects with many areas in machine learning and statistics that study adaptive data collection; most notably, active learning \cite{settles2009active, roy2001toward} and adaptive experimental design \cite{hahn2011adaptive, fedorov2013theory}. We collect data based on model uncertainty, akin to active learning; however, our objective is statistical inference on typically low-dimensional parameters, rather than prediction. Active testing \cite{kossen2021active} also involves adaptive data collection, but it pursues a different objective of high‑precision risk estimation for a fixed model and uses a distinct estimator. Our approach can be seen as an adaptive design assisted by a powerful predictive model, with a robustness wrapper for improved performance.

More distantly, our work also relates to robust statistics and robust machine learning \cite{huber1992robust, zoubir2012robust, staudte2011robust, ramoni2001robust, steinhardt2018robust, cacciarelli2024robust, guo2022robust}. In particular, our method provides a safeguard against poor uncertainty estimation by solving a robust optimization problem \citep{ben2009robust, gabrel2014recent, beyer2007robust}.

\subsection{Problem setup}

We follow the problem setting from \citep{zrnic2024active}. We observe unlabeled instances $X_1, \ldots, X_n$ drawn i.i.d. from a distribution ${P}_X$, but we do not observe their labels $Y_i$. We use ${P}={P}_X \times {P}_{Y \mid X}$ to denote the joint distribution of $(X_i, Y_i)$. Our goal is to perform inference for a parameter $\theta^*$ that depends on the distribution of the unobserved labels; that is, the parameter is a functional of ${P}$. In particular, we assume that $\theta^*$ can be written as:
\[ \theta^*=\underset{\theta}{\arg \min } ~\mathbb{E}\left[\ell_\theta(X, Y)\right], \text { where }(X, Y) \sim {P}.\]
Here, $\ell_{\theta}$ is a convex loss function. This is a broad class of estimands, known as \emph{M-estimation}, and it includes means, medians, linear and logistic regression coefficients, and more.
We have a budget $n_b$ on the number of labels we can collect in expectation, and typically $n_b \ll n$. To assist in imputing the missing labels, we also have a black-box predictive model $f$ at our disposal.

\section{Warm-up: robust sampling for mean estimation}\label{sec:warmup}

Consider the case where $\theta^*$ is the label mean, $\theta^* = \EE[Y]$. The active inference estimator for $\theta^*$ is given by:
\begin{equation}
\label{eq:active_mean_est}
\hat{\theta}^{\pi}=\frac{1}{n} \sum_{i=1}^n\left(f\left(X_i\right)+\left(Y_i-f\left(X_i\right)\right) \frac{\xi_i}{\pi\left(X_i\right)}\right).
\end{equation}
Here, $\pi(\cdot)$ is any sampling rule that satisfies $\EE [\pi(X)] \leq \frac{n_b}{n}$ so that the budget constraint is met on average, and
$\xi_i \sim \text{Bern}(\pi(X_i))$ is the indicator of whether the label $Y_i$ is sampled. Since the number of labeled data points is a sum of independent Bernoullis, a standard Hoeffding argument guarantees that the realized labeling rate will closely match the budget with high probability. Specifically, the labeling ratio will not exceed $\frac{n_b}{n} + \epsilon$ with probability $1- \delta$, provided that $n > \frac{\log(1/\delta)}{2\epsilon^2}$ for any $\epsilon, \delta > 0$. As shown in \citep{zrnic2024active}, the variance of this estimator is 
\begin{equation}
\label{eq:var_meanest}
\Var \left(\hat{\theta}^{\pi}\right)=\frac{1}{n}\left(\Var(Y)+\mathbb{E}\left[(Y-f(X))^2\left(\frac{1}{\pi(X)}-1\right)\right]\right),
\end{equation}
and the optimal sampling rule is $\pi_{\mathrm{opt}}(X_i) \propto \sqrt{\EE[(Y_i - f(X_i))^2|X_i]}$. In other words, it is optimal to upsample where the model $f$ makes the largest errors.

The most straightforward sampling rule that satisfies the budget constraint is the uniform rule: $\pi^{\mathrm{unif}}(X) = {n_b}/{n}$. However, if we have access to a good measure of model uncertainty that can serve as a proxy for the model error $\sqrt{\EE[(Y_i - f(X_i))^2|X_i]}$, then we can obtain a rule that is closer to $\pi_{\mathrm{opt}}$. For example, we might prompt a large language model for its uncertainty about $X_i$ or look at the softmax output of a neural network, and upsample where the uncertainty is high. The issue is that if we severely underestimate the model error, then the estimator's variance can blow up: clearly, if $\pi(X_i)$ is small when the actual error $(Y_i-f(X_i))^2$ is large, the variance will be large as well. This is the reason why we saw poor performance in Figure \ref{fig:teaser}. 

Given any initial sampling rule $\pi$, our approach is to find an improved, \emph{robust} sampling rule $\pi^{\mathrm{robust}}$ that is never worse than either $\pi$ or $\pi^{\mathrm{unif}}$. By that we mean that the resulting active inference estimator will have a variance that is no worse that with either $\pi$ or $\pi^{\mathrm{unif}}$ used for label collection: $\Var(\hat\theta^{\pi^{\mathrm{robust}}}) \leq \min\{\Var(\hat\theta^{\pi}), \Var(\hat\theta^{\pi^{\mathrm{unif}}})\}$.

\subsection{Budget-preserving path}

Since our goal is to find a sampling rule $\pi^{\mathrm{robust}}$ that performs no worse than $\pi^{\mathrm{unif}}$ and an arbitrary given $\pi$, it is natural to consider a path that connects $\pi$ and $\pi^{\mathrm{unif}}$, while preserving the sampling budget along the path.

\begin{Definition}[Budget-preserving path] \label{def: path}
    We call a continuous path $\pi^{(\rho)}$, $\rho\in[0,1]$, a \emph{budget-preserving path} connecting $\pi$ and $\pi^{\mathrm{unif}}$ if $\pi^{(0)} = \pi$, $\pi^{(1)} = \pi^{\mathrm{unif}}$, and $\EE[\pi^{(\rho)}(X)] = \EE[\pi(X)]$ for all $\rho \in [0,1]$.
\end{Definition}
Correspondingly, given a point $\rho$ along the path, we compute the estimator $\hat \theta^{\pi^{(\rho)}}$, obtained as the active inference estimator \eqref{eq:active_mean_est} with sampling rule $\pi^{(\rho)}$. The following are some examples of valid budget-preserving paths.

\begin{example}[Linear path]
\label{ex:linear_path}
$\pi^{(\rho)} = (1-\rho)\pi + \rho \pi^{\mathrm{unif}}$. 
\end{example}

\begin{example}[Geometric path]
\label{ex:geometric_path}
$\pi^{(\rho)} \propto \pi^{1-\rho}(\pi^{\mathrm{unif}})^{\rho}$. The ``$\propto$'' hides the normalization factor that ensures $\EE[\pi^{(\rho)}(X)] = \EE[\pi(X)]$ for all $\rho$.
\end{example}

A natural family of budget-preserving paths can be recovered via the ``least-action'' principle, yielding the definition of geodesic paths. See Appendix \ref{appendix:geodesic} for a general definition of geodesic paths, details of how Examples \ref{ex:linear_path} and \ref{ex:geometric_path} can be recovered as special cases, as well as further examples.

Of course, if we consistently estimate the optimal point $\rho^*\in[0,1]$ along the path, we are guaranteed to find an estimator that outperforms naive active inference and uniform sampling. 
Moreover, the resulting estimator is still asymptotically normal, which permits the construction of valid confidence intervals. We formalize this key result below in which $\sigma_\rho^2 = n \Var(\hat\theta^{\pi^{(\rho)}})$.

\begin{Theorem} \label{thm:path}
    Suppose $\pi^{(\rho)}$ is a budget-preserving path connecting $\pi$ and $\pi^{\mathrm{unif}}$. Let $\rho^* = \mathop{\arg \min}\limits_{\rho} \Var (\hat{\theta}^{\pi^{(\rho)}}) $, and suppose $\hat\rho = \rho^* + o_P(1)$. Then,
\[
\sqrt{n}\left(\hat{\theta}^{\pi^{(\hat\rho)}} -\theta^*\right) \xrightarrow{d} \mathcal{N}\left(0, \sigma_{\rho^*}^2\right),
\]
where $\sigma_{\rho^*}^2 \leq  \min\{\sigma_0^2, \sigma_1^2\}$.
\end{Theorem}

Theorem \ref{thm:path} shows that consistently estimating $\rho^*$ will result in an estimator that is no worse than either endpoint. If $\rho^*$ is additionally unique and within $(0,1)$, then the resulting sampling will strictly outperform both active sampling with $\pi$ and uniform sampling. The theoretical results in this paper are asymptotic; however, validity in the finite-sample regime is shown empirically in the experiments in Section \ref{sec:exp}.

It remains to explain how to estimate $\hat{\rho}$. Recall from \eqref{eq:var_meanest} that $\Var(\hat\theta^{\pi^{(\rho)}}) = \frac{1}{n} \EE[\frac{e^2(X)}{\pi^{(\rho)}(X)}] + C$, where $e^2(X) = \EE[(Y-f(X))^2|X]$ and $C$ is a quantity that has no dependence on $\pi^{(\rho)}$. Therefore, to fit $\hat\rho$, we fit an error function $\hat{e}^2(\cdot)\approx e^2(\cdot)$ and solve for the $\rho$ that minimizes the empirical approximation of $\Var(\hat\theta^{\pi^{(\rho)}})$:
\begin{equation}\label{optimization_objective}
    \hat\rho = \argmin_\rho \frac{1}{n}\sum_{i=1}^{n}\frac{\hat{e}^2(X_i)}{\pi^{(\rho)}(X_i)}.
\end{equation}
 We can find the solution by performing a grid search over $\rho\in[0,1]$. The error $\hat e^2(\cdot)$ can be fit on historical or held-out data, or it can be gradually fine-tuned during the data collection process. Notice that, if the error estimation is consistent in the sense that $\|\hat{e}^2(X) - e^2(X)\|_{\infty}  \mathop{\rightarrow}\limits^{p} 0$ and if $\rho^*$ is unique, then $\hat\rho\mathop{\rightarrow}\limits^{p} \rho^*$, as assumed in Theorem \ref{thm:path}. Here, the assumption that $\hat{e}$
 converges to $e$ follows from classical arguments of uniform approximation of flexible estimators, and is common in the field of semiparametric inference. For instance, the widely-used doubly robust estimator \citep{glynn2010introduction, robins1994estimation}, which is closely related to our estimator, relies on consistent estimation of nuisance functions.

\subsection{Robustness to error function misspecification}

Given a path $\pi^{(\rho)}$, the previous discussion suggests finding $\hat \rho$ that minimizes an empirical approximation of the variance $\Var(\hat\theta^{\pi^{(\rho)}})$. This empirical approximation relies on an error estimate $\hat e(\cdot)$. If this function is severely misspecified, then the computed $\hat \rho$ might be far from $\rho^*$; more importantly, it might not even outperform uniform sampling.

To mitigate this concern, we instead consider a robust optimization problem that incorporates the possibility of $\hat e$ being misspecified:

\begin{equation}
\label{eq:robust_rho}
\rho_{\mathrm{robust}} = \mathop{\arg\min}\limits_{\rho} \max_{\boldsymbol{\epsilon} \in \mathcal C} \frac{1}{n} \sum_{i=1}^{n} \frac{\hat{e}^2(X_i) + \epsilon_i}{\pi^{(\rho)}(X_i)}.
\end{equation}
Here, $\boldsymbol{\epsilon} = (\epsilon_1,\dots,\epsilon_n)$ is the misspecification vector and $\mathcal C$ is the admissible set of misspecifications. This method allows for setting $\pi^{(\rho)}$ close to uniform if the misspecification set $\mathcal C$ is permissive enough.
Solving this minimax problem is computationally efficient, as long as $\mathcal{C}$ is a convex set. The outer problem can be solved via a one-dimensional grid search, while the inner problem is tractable due to convexity.

Now, the question is how we should set $\mathcal C$ in practice. Our default will be to simply use $\mathcal C = \{\boldsymbol{\epsilon} : \|\boldsymbol{\epsilon}\|_2 \leq c\}$, for some hyperparameter $c>0$. Empirically, $c$ can be set by cross-validation. Other choices of the set $\mathcal C$ are possible, such as bounding other norms of $\boldsymbol{\epsilon}$, for example $\|\boldsymbol{\epsilon}\|_1<c$. Empirically we found the $\ell_2$ norm to work the best, and in illustrative theoretical examples we reach the same conclusion; see Appendix \ref{appendix:constraint} for details. We also tried relative misspecification, in the sense that $\epsilon_i = \hat{e}^2(X_i)(1 + \eta_i)$, and constrained either the $\ell_1$ or $\ell_2$ norm of the relative perturbation $\eta$. We found that this does not perform as well.

\citet{zrnic2024active} briefly discussed a robustness proposal with linear interpolation. It assumes access to historical data, and otherwise it selects a default value for the coefficient, which has no guarantee to outperform uniform and active sampling. Our analysis is far more thorough and systematic, expanding the set of interpolating paths, not requiring historical data but incorporating a burn-in period, and adding a robustness constraint. These are all crucial for the practicality and reliability of the method; see Section \ref{sec:exp} for details.

There are other potential optimization objectives to take into  account robustness constraints. For example, one may penalize small values of $\rho$ in the objective \eqref{optimization_objective} with regularization, and similarly use cross-validation to choose the penalty parameter. We leave the investigation of such alternatives for future work.

\section{Robust sampling for general M-estimation}\label{sec:M_est}

Our sampling principle can be directly extended to general convex M-estimation, as considered in \cite{zrnic2024active}. We explain this step-by-step for completeness.

Recall that  we consider all inferential targets of the form $\theta^*=\arg \min _\theta \mathbb{E}\left[\ell_\theta(X, Y)\right]$, for a convex loss $\ell_\theta$. Denote $\ell_{\theta, i}=$ $\ell_\theta\left(X_i, Y_i\right), \ell_{\theta, i}^f=\ell_\theta\left(X_i, f\left(X_i\right)\right)$, and define $\nabla \ell_{\theta, i}$ and $\nabla \ell_{\theta, i}^f$ similarly. For an active sampling strategy $\pi$, the general active inference estimator is defined as:
\begin{equation}\label{path_m_est}
    \hat{\theta}^{\pi}=\underset{\theta}{\arg \min } ~L^\pi(\theta), \text { where } L^\pi(\theta)=\frac{1}{n} \sum_{i=1}^n\left(\ell_{\theta, i}^f+\left(\ell_{\theta, i}-\ell_{\theta, i}^f\right) \frac{\xi_i}{\pi(X_i)}\right).
\end{equation}
As before, $\xi_i \sim \text{Bern}(\pi(X_i))$ is the indicator of whether the label $Y_i$ is sampled. Following \cite{zrnic2024active}, we know that the asymptotic covariance matrix of $\hat\theta^\pi$ equals:
\[\Sigma^\pi = H_{\theta^*}^{-1} \operatorname{Var}\left(\nabla \ell_{\theta^*}^f+\left(\nabla \ell_{\theta^*}-\nabla \ell_{\theta^*}^f\right) \frac{\xi}{\pi(X)}\right) H_{\theta^*}^{-1},\]
where $H_{\theta^*}$ is the Hessian $H_{\theta^*}=\nabla^2 \mathbb{E}\left[\ell_{\theta^*}(X, Y)\right]$.

We again consider budget-preserving paths $\pi^{(\rho)}$ and tune the parameter $\rho$ such that we minimize the variance of the resulting estimator $\hat\theta^{\pi^{(\rho)}}$. Denote by $\Sigma_0$ and $\Sigma_1$ the asymptotic covariance matrices of the active inference estimator \eqref{path_m_est} using $\pi^{(0)} = \pi$ and $\pi^{(1)} = \pi^{\mathrm{unif}}$, respectively.

\begin{Theorem} \label{thm:path_m}
    Suppose $\pi^{(\rho)}$ is a budget-preserving path connecting $\pi$ and $\pi^{\mathrm{unif}}$. Given a coordinate $j$ of interest, let $\rho^* = \mathop{\arg \min}\limits_{\rho} \Sigma^{\pi^{(\rho)}}_{jj}$, and suppose $\hat{\rho} = \rho^* + o_P(1)$. Suppose further that $\hat{\theta}^{\pi^{(\rho^*)}} \xrightarrow{p} \theta^*$. Then, \[\sqrt{n}\left(\hat{\theta}^{\pi^{(\hat{\rho})}} -\theta^*\right) \xrightarrow{d} \mathcal{N}\left(0, \Sigma_{\rho^*}\right),\]
where $\Sigma_{\rho^*,jj} \leq \min\{\Sigma_{0,jj}, \Sigma_{1,jj}\}$.
\end{Theorem}

The consistency condition $\hat{\theta}^{\pi^{(\rho^*)}} \xrightarrow{p} \theta^*$ is standard; see the corresponding discussion in \cite{zrnic2024active} and \cite{angelopoulos2023ppi++}. For example, it is ensured when $L^\pi$ is convex, such as in the case of generalized linear models (GLMs), or when the parameter space is compact.

As in the case of mean estimation, we fit $\hat\rho$ by approximating the variance of the estimator $\Sigma^{\pi^{(\rho)}}$ and searching over $\rho$. However, here the notion of error $e^2(\cdot)$ we need to estimate is different. In particular, given the form of $\Sigma^\pi$, we let
\[\hat{\rho} = \mathop{\arg \min}\limits_{\rho} \frac{1}{n}\sum_{i=1}^{n} \frac{\hat{e}^2(X_i)}{\pi^{(\rho)}(X_i)},\] 
where $\hat{e}^2(X)$ aims to approximate $e^2(X) = \EE[((\nabla \ell_{\theta^*}-\nabla \ell_{\theta^*}^f)^{\top} h^{(j)})^2|X]$ and $h^{(j)}$ is the $j$-th column of $H_{\theta^*}^{-1}$. In the context of generalized linear models (GLMs), this error simplifies to $e^2(X) = \EE[(Y-f(X))^2|X] \cdot (X^\top h^{(j)})^2$. Therefore, as for mean estimation, the problem essentially reduces to estimating the error $\EE[(Y-f(X))^2|X]$. As before, if $\hat e^2$ consistently estimates $e^2$, then $\hat \rho$ consistently estimates $\rho^*$.

Finally, to protect against poorly estimated errors $\hat e$, we can incorporate an uncertainty set $\mathcal C$ around the error estimates just as before \eqref{eq:robust_rho}. Again, the only difference here is that the $\hat e^2(X_i)$'s are estimating a different notion of model error tailored to the inference problem at hand.

We summarize our general \emph{robust active inference} algorithm in Algorithm \ref{alg_path_robust}.

\begin{algorithm}[H]\label{alg_path_robust}
    \caption{Robust Active Inference}
    \KwIn{unlabeled data $X_1,\ldots, X_n$, labeling budget $n_b$, predictive model $f$, initial sampling rule $\pi$, budget-preserving path $\pi^{(\rho)}$, error estimator $\hat{e}^2(\cdot)$, robustness constraint $\mathcal{C}$}
    Solve the minimax problem $\rho_{\mathrm{robust}} = \mathop{\arg\min}\limits_{\rho\in[0,1]} \mathop{\max}\limits_{\boldsymbol{\epsilon} \in \mathcal{C}} \frac{1}{n}\sum_{i=1}^{n} \frac{\hat{e}^2(X_i) + \epsilon_i}{\pi^{(\rho)}(X_i)}$\\
    Sample labeling decisions according to $\pi^{(\rho_{\mathrm{robust}})}(X_i)$: $\xi_i \sim \operatorname{Bern}\left(\pi^{(\rho_{\mathrm{robust}})}(X_i)\right), i \in[n]$\\
    Collect labels $\left\{Y_i: \xi_i=1\right\}$\\
    \KwOut{estimator $\hat{\theta}^{\pi^{(\rho_{\mathrm{robust}})}}=\underset{\theta}{\arg \min } ~L^{\pi^{(\rho_{\mathrm{robust}})}}$, as defined in Eq.~\eqref{path_m_est}}
\end{algorithm}

\section{Experiments}\label{sec:exp}

We turn to evaluating the performance of our robust sampling approach empirically. Each of the following subsections is dedicated to a different experiment using social science research data.
Section \ref{sec:post-election-survey} measures presidential approval, Section \ref{sec:census-data} analyzes US age–income patterns, and Section \ref{sec:css-llm} applies language models to score text on social attributes such as political bias.
On each of these datasets, we use the following methods to collect labels: (1) uniform sampling, which essentially recovers prediction-powered inference \cite{angelopoulos2023prediction}; (2) standard uncertainty-based active sampling~\cite{zrnic2024active}; and (3) our robust active method as per Algorithm~\ref{alg_path_robust}.
Each dataset will use a different base predictive model $f$, which we describe therein. We set the target coverage level to be $0.9$ throughout. 

The main metric used for the comparison is effective sample size. To define this metric formally, consider the baseline estimator that samples uniformly at random, i.e., according to $\pi^{\mathrm{unif}}$. Its effective sample size is simply its budget $n_b$. For other estimators, we say that the effective sample size is equal to  $n_{\mathrm{eff}}$ if the estimator achieves the same variance as the baseline estimator with budget $n_{\mathrm{eff}}$. For example, if given budget $n_b=100$ the estimator achieves the same variance as the baseline estimator with double the budget, then the estimator has $n_{\mathrm{eff}} = 200$. A larger $n_{\mathrm{eff}}$ indicates a more efficient estimator. In the case where the effective sample size falls below the budget, $n_{\mathrm{eff}} < n_b$, the estimator performs worse than the baseline. We show one standard deviation around the effective sample size in all plots, estimated over $500$ trials.

We also plot empirical estimates of the methods' coverage. We estimate the coverage by resampling the data, constructing confidence intervals for each resampling, and calculating the proportion of times the true parameter value (approximated by the full-data estimate of the target $\theta^*$) falls within the constructed intervals. 
This approach allows us to assess how reliably each method achieves the target coverage level. We resample 500 times to estimate the coverage. (We note that this approach yields conservative coverage estimates when $n_b$ is large, because we have $n-n_b$ ``fresh'' labels to approximate $\theta^*$.) From the theory, we know that the coverage should be exactly $0.9$ for all baselines.

\subsection{Post-election survey research}
\label{sec:post-election-survey}

Following \citep{zrnic2024active}, we evaluate the different methods on survey data collected by the Pew Research Center following the 2020 United States presidential election, aiming at gauging people’s approval of the presidential candidates’ political messaging \cite{atp79}. We aim to estimate the approval rate $\theta^* = \EE[Y]$, where $Y\in\{0,1\}$ is a binary indicator of approval of Biden's political messaging.
We use a multilayer perceptron (MLP) as our predictive model $f$. At the beginning, we have a ``burn-in'' period where we collect all burn-in labels $Y_i$ and we use this burn-in data to estimate the error function $\hat e(\cdot)$. Afterwards, we use the fitted function to run robust active inference, as per Algorithm~\ref{alg_path_robust}. Naturally, the burn-in period counts towards the overall labeling budget $n_b$.

We study three questions: (1) the effect of tuning $\rho$ along the budget-preserving path, without incorporating a robustness constraint $\mathcal C$; (2) the effect of tuning $\rho$ along the path and the robust optimization over $\mathcal C$ combined; and (3) the performance of different budget-preserving paths.

\begin{figure}[t]
    \centering
    \includegraphics[width=0.9\linewidth]{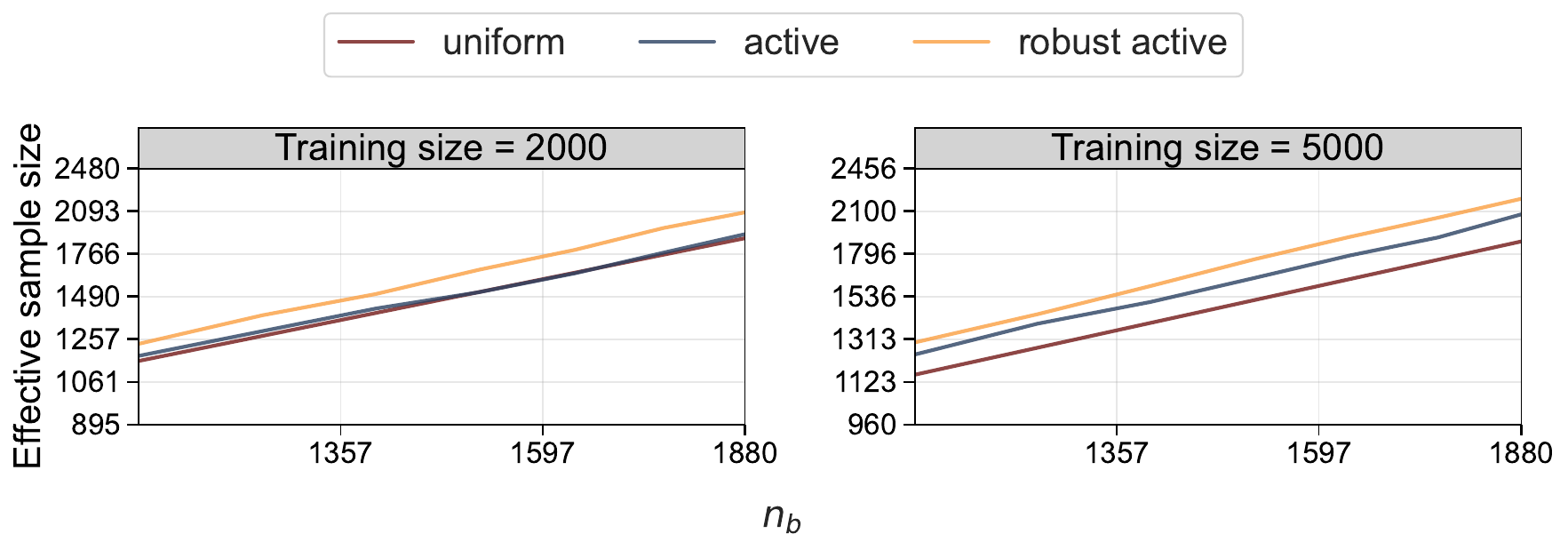}
   \caption{\textbf{Effective sample size on Pew post-election survey data}, for different dataset sizes used to train $f$. We compare uniform, active, and robust active sampling, for different values of the sampling budget $n_b$. The target of inference is the approval rate of a presidential candidate. We show the mean and one standard deviation (see Appendix \ref{appendix:plots}) of the effective sample size estimated over $500$ trials; in each trial we independently sample the observed labels.} 
    \label{fig:election_without_robust}
\end{figure}

\paragraph{Tuning along the budget-preserving path.}
First, we conduct an experiment without the robustness set $\mathcal C$, only tuning the parameter $\hat \rho$ along the budget-preserving path. We choose the geometric path from Example \ref{ex:geometric_path}. To implement active inference, we use $\pi(x) \propto \min\{f(x), 1-f(x)\}$, in which $f(x)$ is the predicted probability that the label takes on the value $1$, as considered in \cite{zrnic2024active}. See Figure~\ref{fig:election_without_robust} for the results. We consider two training dataset sizes used to train $f$, allowing us to see the results for a less accurate $f$ (left) and a more accurate one (right).
We find that, even without robust optimization but only optimizing along the budget-preserving path, robust active inference can lead to noticeable improvements in terms of power compared to naive uncertainty-based active sampling and uniform sampling. The performance of standard active inference crucially depends on the quality of $f$ and its uncertainties. We defer the corresponding coverage plots to Appendix \ref{appendix:more_exp}.

\paragraph{Incorporating robustness.}
One strategy proposed by \citet{zrnic2024active} is to estimate $\hat e$ and set $\pi$ proportional to $\hat e$. With this choice, without the additional step of robust optimization, our robust sampling approach would trivially estimate $\hat\rho = 0$ (a proof of this claim can be found in Appendix~\ref{appendix:proofs}). We show that incorporating the robustness constraint resolves this issue when $\pi(x) \propto \hat e(x)$. As in the previous case, we use the geometric path and an MLP as the predictive model. The results are shown in Figure \ref{fig:election_robust}. Recall, $\hat e$ is estimated from the burn-in data. Thus, the longer the burn-in period, the better the fit $\hat e$. This is consistent with the observation that active inference gradually outperforms uniform sampling as the burn-in period grows. However, when there is little data to fit $\hat e$, active sampling leads to a significantly higher variance than uniform sampling. Our robust sampling approach is never worse than either baseline, across all burn-in data sizes. This is explained by the fact that, when the fit $\hat e$ is poor, the constraint set $\mathcal C$ chosen via cross-validation is large, resulting in a large $\rho_{\mathrm{robust}}$, thus pushing the sampling rule closer to uniform. In Figure \ref{fig:rho_robust} (left), we plot the optimized value $\rho_{\mathrm{robust}}$ for different burn-in sizes. 
As expected, $\rho_{\mathrm{robust}}$ decreases, which means that the optimal strategy gradually moves from uniform sampling toward standard active sampling as the quality of $\hat e$ improves.

\begin{figure}[!htbp]
    \centering
    \includegraphics[width=1\linewidth]{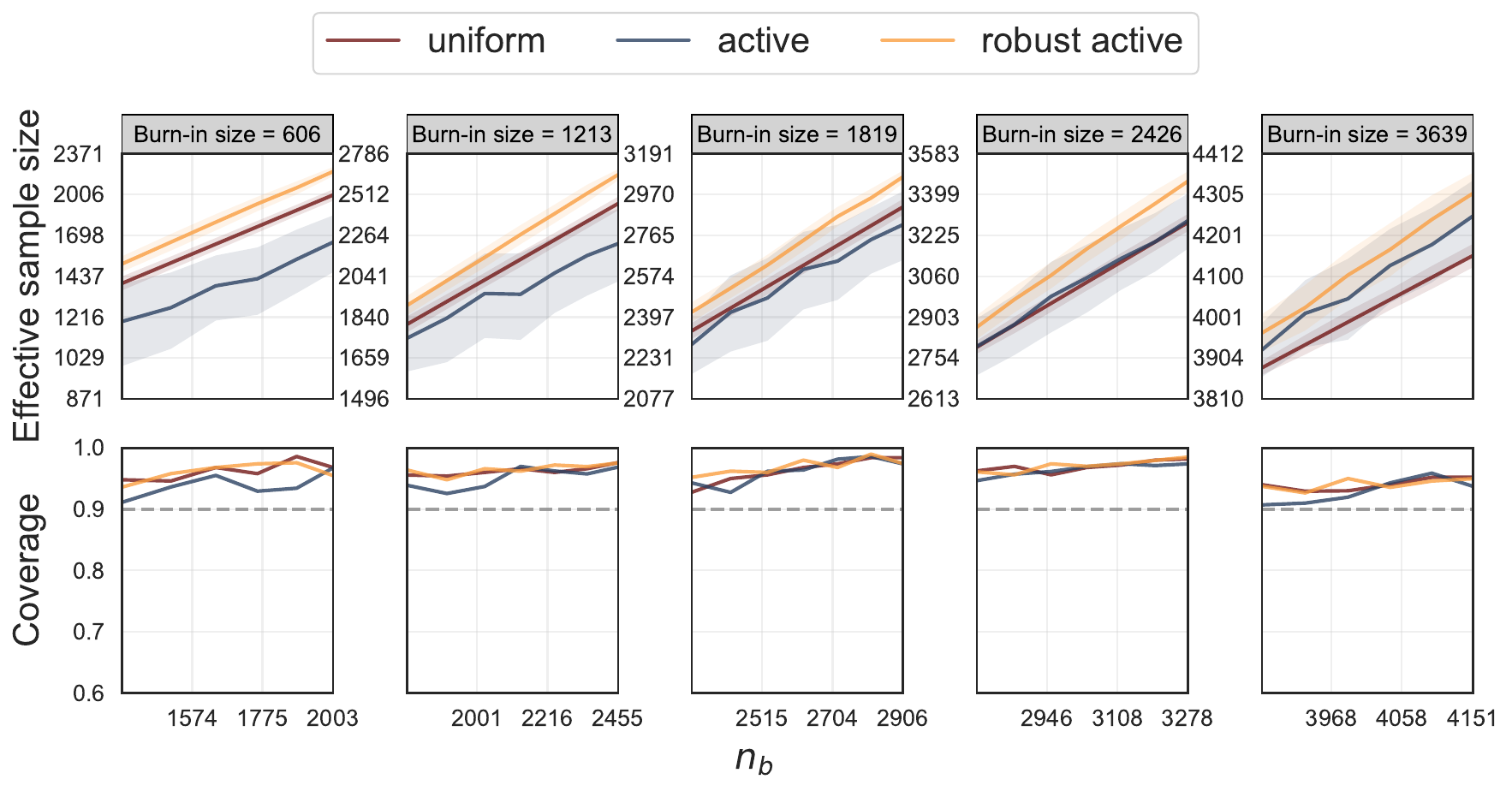}
    \caption{
    \textbf{Effective sample size (top) and coverage (bottom) on Pew post-election survey data}, for varying burn-in dataset sizes with respect to different proportions of the data. We compare uniform, active, and robust active sampling, for different values of the sampling budget $n_b$. The target of inference is the approval rate of a presidential candidate. We show the mean and one standard deviation of the effective sample size estimated over $500$ trials; in each trial we independently sample the observed labels.}
    \label{fig:election_robust}
\end{figure}

\begin{figure}[!htbp]
    \centering
    \includegraphics[width=0.9\linewidth]{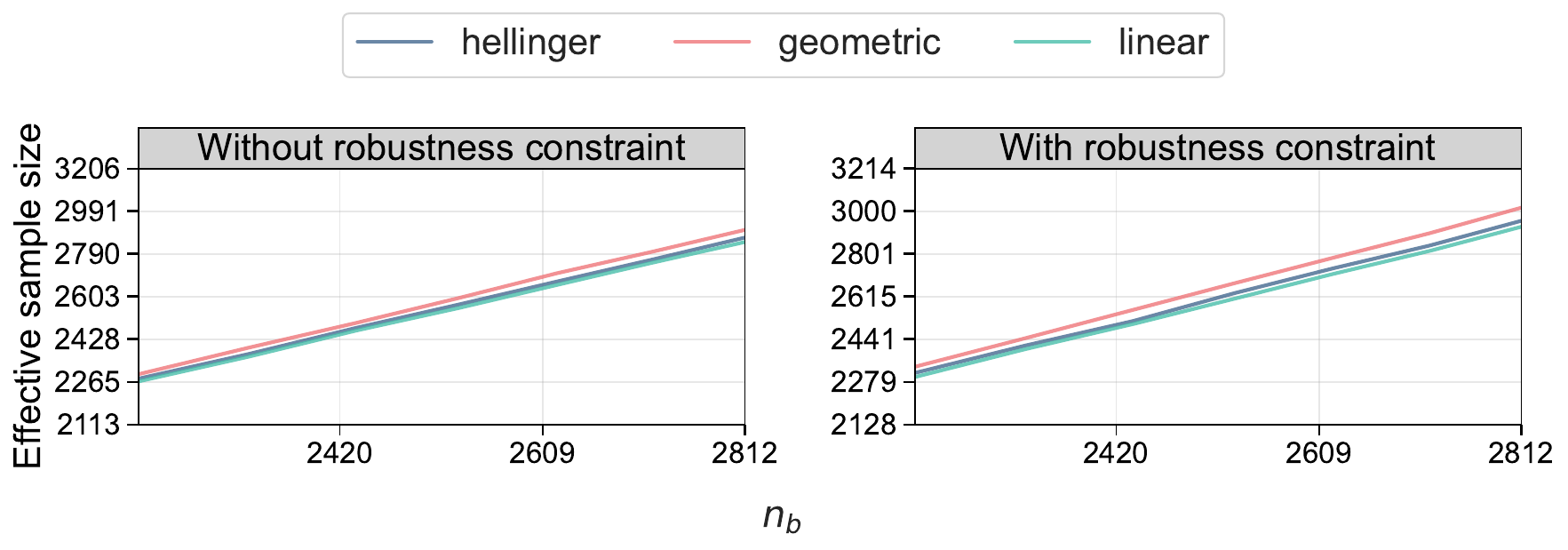}
    \caption{\textbf{Effective sample size for different budget-preserving paths on Pew post-election survey data}, without (left) and with (right) a robustness constraint $\mathcal C$. In both cases, the geometric path leads to the largest effective sample size. The target of inference is the same as in Figure \ref{fig:election_robust}. We show the mean and one standard deviation (see Appendix \ref{appendix:plots}) of the effective sample size estimated over $500$ trials; in each trial we independently sample the observed labels.} 
    \label{fig:guide_ess}
\end{figure}

\paragraph{Choice of budget-preserving path.}
We have thus far used the geometric path as our budget-preserving path. In Figure~\ref{fig:guide_ess} we compare three budget-preserving paths: the linear path, the geometric path, and the Hellinger path (see Appendix \ref{appendix:geodesic}). On the post-election survey dataset, Figure~\ref{fig:guide_ess} shows that the geometric path is the best of the three chosen paths, regardless of whether or not robust optimization over $\mathcal C$ is used. Therefore, as a practical default, we recommend using the geometric path. It has been stress-tested and has consistently demonstrated strong performance in our evaluations. We believe this is a good tradeoff between simplicity and performance. For improved performance with a better choice of path, the practitioner might want to tune it in a data-driven way; for example, based on the estimated variance on a small held-out dataset.

\subsection{Census data analysis}
\label{sec:census-data}

We study the annual American Community Survey (ACS) Public Use Microdata Sample (PUMS) collected
 by the US Census Bureau \cite{ding2021retiring}. We are interested in investigating the relationship between age and income in survey data collected in California in 2019, controlling for sex. We estimate the age coefficient $\theta^*$ of the linear regression vector when regressing income on age and sex. We use an XGBoost model \cite{chen2016xgboost} to predict income $Y$ from available demographic covariates. As in the previous problem, we set $\pi(x) \propto \hat e(x)$, as in \cite{zrnic2024active}, and fit $\hat e$ using burn-in data. We use the geometric path and incorporate the robust optimization over $\mathcal C$. We show the results in Figure \ref{fig:census}. Again, we observe that the robust approach outperforms both standard active sampling and uniform sampling for different qualities of the error estimate $\hat e(\cdot)$, corresponding to different burn-in dataset sizes. Standard active inference, on the other hand, is very sensitive to the quality of $\hat e$. In Figure \ref{fig:rho_robust} (right), we plot the optimized value $\rho_{\mathrm{robust}}$ for different burn-in sizes. As in the previous example, $\rho_{\mathrm{robust}}$ decreases as the quality of $\hat e$ improves, as expected. We include corresponding coverage plots in Appendix \ref{appendix:more_exp}.

\begin{figure}[t]
    \centering
    \includegraphics[width=1\linewidth]{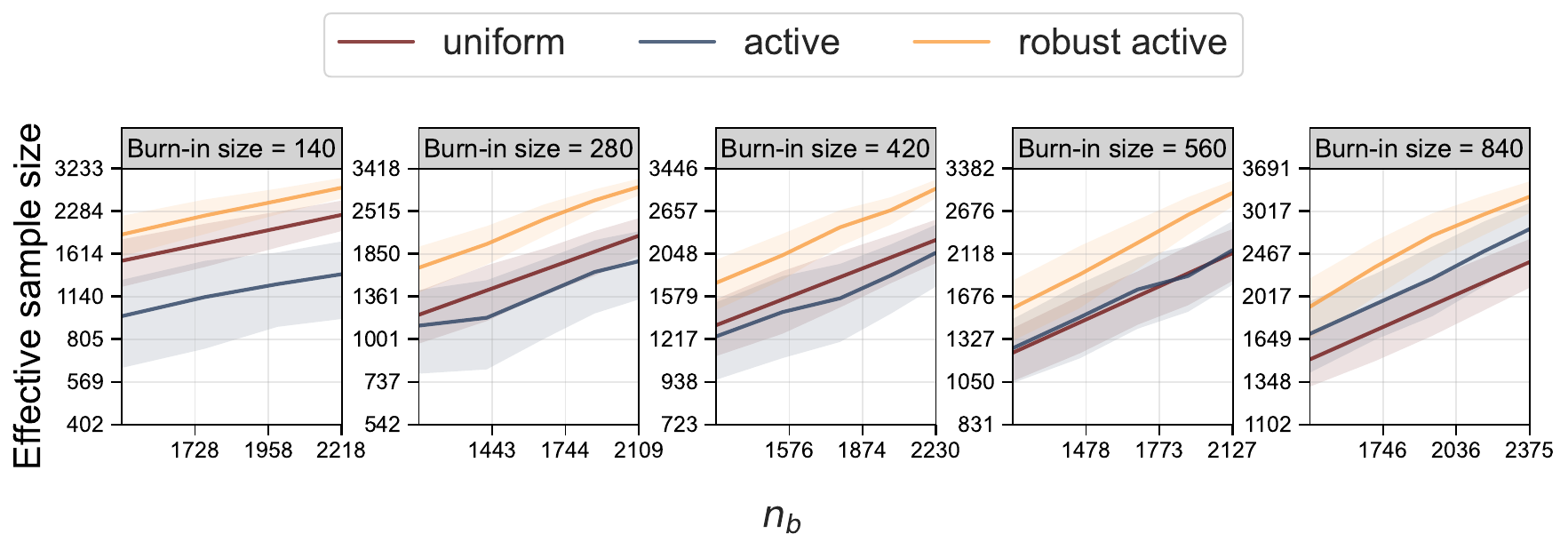}
    \caption{
    \textbf{Effective sample size on US Census data}, for varying burn-in dataset sizes. We compare uniform, active, and robust active sampling, for different values of the sampling budget $n_b$. The target of inference is the relationship between age and income, estimated via a linear regression. We show the mean and one standard deviation of the effective sample size estimated over $500$ trials; in each trial we independently sample the observed labels.
    }
    \label{fig:census}
\end{figure}

\begin{figure}[t]
    \centering
    \includegraphics[width=0.9\linewidth]{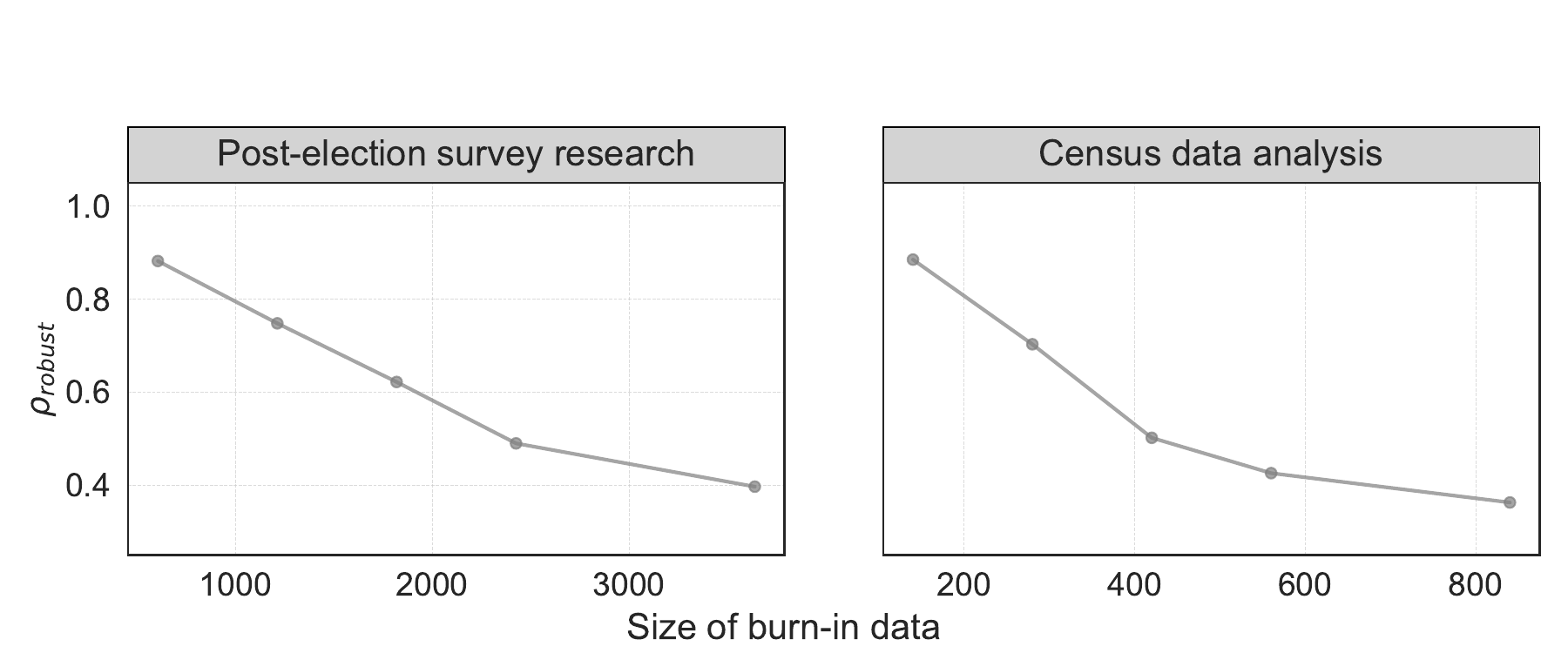}
       \caption{\textbf{Optimized value $\rho_{\mathrm{robust}}$ along the geometric path} as a function of the size of the burn-in data for the post-election survey data (left) and US Census data (right).}
    \label{fig:rho_robust}
\end{figure}

\subsection{Computational social science with language models}
\label{sec:css-llm}

We study three text annotation tasks used for computational social science research. In each task, we have text instances $X_i$ and we seek to collect labels $Y_i$ related to the text's sentiment, political leaning, and so on. We wish to use a large language model (LLM) $f$ to predict the high-quality annotations $Y_i$, which are typically collected through laborious human annotation. A natural way of actively sampling human annotations is according to the confidence of the language model \cite{gligoric2024can, li2023coannotating}. \citet{tian2023just} propose prompting LLMs to verbalize their confidence in the provided answer, and they find that this results in fairly calibrated confidence scores. \citet{gligoric2024can} find that such scores can be useful in actively sampling human annotations.
We use GPT-4o annotations and confidences collected by \citet{gligoric2024can}. We apply active inference with $\pi(X_i)\propto (1-C_i)$, where $C_i$ is the collected confidence score of the language model for prompt $X_i$. This can be a brittle strategy, since the scores are often overconfident and thus result in very small sampling probabilities, which can blow up the estimator variance through inverse probability weighting. For robust active inference, we use the geometric path and robust optimization with an $\ell_2$ constraint set $\mathcal C$, as before.

\paragraph{Political bias.}
In the first task, the goal is to study the political leaning of media articles, using the data curated by \citet{baly2020we}. The labels $Y$ are one of \texttt{left}, \texttt{centrist}, or \texttt{right}. The inferential target is the prevalence of right-leaning articles:
$\theta^*=\mathbb{E}[\mathbf{1}\{Y= \texttt{right}\}]$.

\paragraph{Politeness.}
The next task is to estimate how certain linguistic devices impact the perceived politeness of online requests. We use the dataset of requests from Wikipedia and StackExchange curated by \citet{danescu2013computational}. We study how the presence of hedging in the request, $X_{\mathrm{hedge}}\in\{0,1\}$, impacts whether a text is seen as polite, $Y\in\{0,1\}$. Formally, $\theta^*$ is this effect estimated via a logistic regression with an intercept: $\operatorname{logit}\left(\mathbb{P}\left(Y=1 \mid X_{\text {hedge }}\right)\right)=\theta_0+$ $\theta^* X_{\text {hedge }}$.

\paragraph{Misinformation.}
Finally, we study the prevalence of misinformation in news headlines, using the dataset collected by \citet{gabrel2014recent}. The labels $Y\in\{0,1\}$ indicate whether a headline contains misinformation. The inferential target is the prevalence of misinformation, $\theta^* = \EE[Y]$.

\begin{figure}[t]
    \centering
    \includegraphics[width=0.9\linewidth]{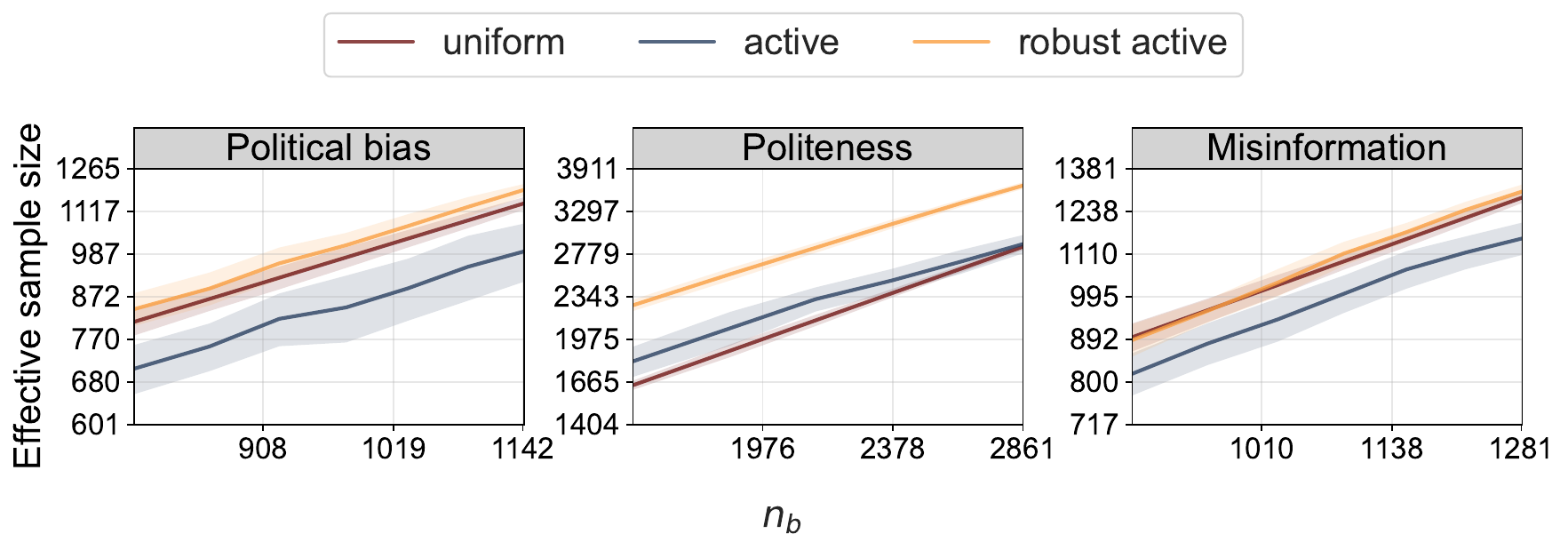}
    \caption{\textbf{Effective sample size on social science text annotation datasets}. We compare uniform, active, and robust active sampling, for different values of the sampling budget $n_b$. The targets of inference are (left to right) the prevalence of right-leaning political bias, the relationship between hedging and politeness, and the prevalence of misinformation. We show the mean and one standard deviation of the effective sample size estimated over $500$ trials; in each trial we independently sample the observed labels.}
    \label{fig:llm_ess}
\end{figure}

We show the results in Figure \ref{fig:llm_ess}. Across all tasks, the robust approach is essentially never worse than uniform sampling or active inference, in cases even outperforming both by a large margin. Standard active inference often leads to large intervals, given that sampling directly according to the model's verbalized uncertainty leads to instability through inverse probability weighting. 
We include the corresponding coverage plots in Appendix \ref{appendix:more_exp}.

\section{Conclusion}\label{sec:conclusion}
We presented robust sampling strategies for active inference: a principled hedge between uniform and conventional active sampling. By selecting an optimal tuning parameter $\rho$ along a budget-preserving path, robust active inference ensures performance that is no worse than with standard active sampling, and it reduces to near-uniform sampling when uncertainty scores are unreliable.
Furthermore, the estimator can even surpass standard active inference given reliable uncertainties.

Many directions remain for future work. For example, it would be valuable to understand how to optimally choose the constraint set $\mathcal C$, or at least how to choose between several different constraint sets. As presented, our procedure is sensitive to the choice of $\mathcal C$ and may result in sampling rules that are too close or too far from uniform if this set is chosen poorly. We also leave investigations into the optimal budget-preserving path, and practical heuristics for how a practitioner might effectively choose a good path in a data-driven way, for future work.

\section*{Acknowledgement}
EJC was supported by the Office of Naval Research grant N00014-24-1-2305, the National Science Foundation grant DMS-2032014, and the Simons Foundation under award 814641. 

\bibliography{ref}

@inproceedings{zrnic2024active,
  title={Active statistical inference},
  author={Zrnic, Tijana and Cand{\`e}s, Emmanuel J},
  booktitle={Proceedings of the 41st International Conference on Machine Learning},
  pages={62993--63010},
  year={2024}
}

@inproceedings{xie2016transfer,
  title={Transfer learning from deep features for remote sensing and poverty mapping},
  author={Xie, Michael and Jean, Neal and Burke, Marshall and Lobell, David and Ermon, Stefano},
  booktitle={Proceedings of the AAAI conference on artificial intelligence},
  volume={30},
  year={2016}
}

@article{jumper2021highly,
  title={Highly accurate protein structure prediction with AlphaFold},
  author={Jumper, John and Evans, Richard and Pritzel, Alexander and Green, Tim and Figurnov, Michael and Ronneberger, Olaf and Tunyasuvunakool, Kathryn and Bates, Russ and {\v{Z}}{\'\i}dek, Augustin and Potapenko, Anna and others},
  journal={nature},
  volume={596},
  number={7873},
  pages={583--589},
  year={2021},
  publisher={Nature Publishing Group}
}

@article{angelopoulos2023prediction,
  title={Prediction-powered inference},
  author={Angelopoulos, Anastasios N and Bates, Stephen and Fannjiang, Clara and Jordan, Michael I and Zrnic, Tijana},
  journal={Science},
  volume={382},
  number={6671},
  pages={669--674},
  year={2023},
  publisher={American Association for the Advancement of Science}
}

@article{angelopoulos2023ppi++,
  title={{PPI}++: Efficient prediction-powered inference},
  author={Angelopoulos, Anastasios N and Duchi, John C and Zrnic, Tijana},
  journal={arXiv preprint arXiv:2311.01453},
  year={2023}
}

@article{zrnic2024cross,
  title={Cross-prediction-powered inference},
  author={Zrnic, Tijana and Cand{\`e}s, Emmanuel J},
  journal={Proceedings of the National Academy of Sciences},
  volume={121},
  number={15},
  pages={e2322083121},
  year={2024},
  publisher={National Acad Sciences}
}

@article{seguy2015principal,
  title={Principal geodesic analysis for probability measures under the optimal transport metric},
  author={Seguy, Vivien and Cuturi, Marco},
  journal={Advances in Neural Information Processing Systems},
  volume={28},
  year={2015}
}

@article{liero2016optimal,
  title={Optimal transport in competition with reaction: The Hellinger--Kantorovich distance and geodesic curves},
  author={Liero, Matthias and Mielke, Alexander and Savar{\'e}, Giuseppe},
  journal={SIAM Journal on Mathematical Analysis},
  volume={48},
  number={4},
  pages={2869--2911},
  year={2016},
  publisher={SIAM}
}

@article{ben2009robust,
  title={Robust Optimization},
  author={Ben-Tal, A},
  journal={Princeton University Press},
  volume={2},
  pages={35--53},
  year={2009}
}

@article{gabrel2014recent,
  title={Recent advances in robust optimization: An overview},
  author={Gabrel, Virginie and Murat, C{\'e}cile and Thiele, Aur{\'e}lie},
  journal={European journal of operational research},
  volume={235},
  number={3},
  pages={471--483},
  year={2014},
  publisher={Elsevier}
}

@article{beyer2007robust,
  title={Robust optimization--a comprehensive survey},
  author={Beyer, Hans-Georg and Sendhoff, Bernhard},
  journal={Computer methods in applied mechanics and engineering},
  volume={196},
  number={33-34},
  pages={3190--3218},
  year={2007},
  publisher={Elsevier}
}

@incollection{huber1992robust,
  title={Robust estimation of a location parameter},
  author={Huber, Peter J},
  booktitle={Breakthroughs in statistics: Methodology and distribution},
  pages={492--518},
  year={1992},
  publisher={Springer}
}

@article{zoubir2012robust,
  title={Robust estimation in signal processing: A tutorial-style treatment of fundamental concepts},
  author={Zoubir, Abdelhak M and Koivunen, Visa and Chakhchoukh, Yacine and Muma, Michael},
  journal={IEEE Signal Processing Magazine},
  volume={29},
  number={4},
  pages={61--80},
  year={2012},
  publisher={IEEE}
}

@book{staudte2011robust,
  title={Robust estimation and testing},
  author={Staudte, Robert G and Sheather, Simon J},
  year={2011},
  publisher={John Wiley \& Sons}
}

@article{ramoni2001robust,
  title={Robust learning with missing data},
  author={Ramoni, Marco and Sebastiani, Paola},
  journal={Machine Learning},
  volume={45},
  pages={147--170},
  year={2001},
  publisher={Springer}
}

@book{steinhardt2018robust,
  title={Robust learning: Information theory and algorithms},
  author={Steinhardt, Jacob},
  year={2018},
  publisher={Stanford University}
}

@article{cacciarelli2024robust,
  title={Robust online active learning},
  author={Cacciarelli, Davide and Kulahci, Murat and Tyssedal, John S{\o}lve},
  journal={Quality and Reliability Engineering International},
  volume={40},
  number={1},
  pages={277--296},
  year={2024},
  publisher={Wiley Online Library}
}

@inproceedings{guo2022robust,
  title={Robust active learning: Sample-efficient training of robust deep learning models},
  author={Guo, Yuejun and Hu, Qiang and Cordy, Maxime and Papadakis, Mike and Traon, Yves Le},
  booktitle={Proceedings of the 1st International Conference on AI Engineering: Software Engineering for AI},
  pages={41--42},
  year={2022}
}

@inproceedings{chen2016xgboost,
  title={Xgboost: A scalable tree boosting system},
  author={Chen, Tianqi and Guestrin, Carlos},
  booktitle={Proceedings of the 22nd acm sigkdd international conference on knowledge discovery and data mining},
  pages={785--794},
  year={2016}
}

@article{gligoric2024can,
  title={Can Unconfident LLM Annotations Be Used for Confident Conclusions?},
  author={Gligori{\'c}, Kristina and Zrnic, Tijana and Lee, Cinoo and Cand{\`e}s, Emmanuel J and Jurafsky, Dan},
  journal={arXiv preprint arXiv:2408.15204},
  year={2024}
}

@article{ji2025predictions,
  title={Predictions as surrogates: Revisiting surrogate outcomes in the age of ai},
  author={Ji, Wenlong and Lei, Lihua and Zrnic, Tijana},
  journal={arXiv preprint arXiv:2501.09731},
  year={2025}
}

@article{kluger2025prediction,
  title={Prediction-Powered Inference with Imputed Covariates and Nonuniform Sampling},
  author={Kluger, Dan M and Lu, Kerri and Zrnic, Tijana and Wang, Sherrie and Bates, Stephen},
  journal={arXiv preprint arXiv:2501.18577},
  year={2025}
}

@article{ao2024prediction,
  title={Prediction-Guided Active Experiments},
  author={Ao, Ruicheng and Chen, Hongyu and Simchi-Levi, David},
  journal={arXiv preprint arXiv:2411.12036},
  year={2024}
}

@article{zrnic2024note,
  title={A note on the prediction-powered bootstrap},
  author={Zrnic, Tijana},
  journal={arXiv preprint arXiv:2405.18379},
  year={2024}
}

@book{burago2001course,
  title={A course in metric geometry},
  author={Burago, Dmitri and Burago, Yuri and Ivanov, Sergei},
  volume={33},
  year={2001},
  publisher={American Mathematical Society Providence}
}

@misc{atp79,
   author = {Pew},
   year = 2020,
   title = {American Trends Panel ({ATP}) Wave 79},
url = {https://www.pewresearch.org/science/dataset/american-trends-panel-wave-79/}
}

@article{chhikara2025mind,
  title={Mind the confidence gap: Overconfidence, calibration, and distractor effects in large language models},
  author={Chhikara, Prateek},
  journal={arXiv preprint arXiv:2502.11028},
  year={2025}
}

@article{zhang2024calibrating,
  title={Calibrating the confidence of large language models by eliciting fidelity},
  author={Zhang, Mozhi and Huang, Mianqiu and Shi, Rundong and Guo, Linsen and Peng, Chong and Yan, Peng and Zhou, Yaqian and Qiu, Xipeng},
  journal={arXiv preprint arXiv:2404.02655},
  year={2024}
}

@article{zhou2023navigating,
  title={Navigating the grey area: How expressions of uncertainty and overconfidence affect language models},
  author={Zhou, Kaitlyn and Jurafsky, Dan and Hashimoto, Tatsunori},
  journal={arXiv preprint arXiv:2302.13439},
  year={2023}
}

@article{fisch2024stratified,
  title={Stratified prediction-powered inference for hybrid language model evaluation},
  author={Fisch, Adam and Maynez, Joshua and Hofer, R Alex and Dhingra, Bhuwan and Globerson, Amir and Cohen, William W},
  journal={arXiv preprint arXiv:2406.04291},
  year={2024}
}

@article{miao2023assumption,
  title={Assumption-lean and data-adaptive post-prediction inference},
  author={Miao, Jiacheng and Miao, Xinran and Wu, Yixuan and Zhao, Jiwei and Lu, Qiongshi},
  journal={arXiv preprint arXiv:2311.14220},
  year={2023}
}

@article{miao2024task,
  title={Task-agnostic machine-learning-assisted inference},
  author={Miao, Jiacheng and Lu, Qiongshi},
  journal={arXiv preprint arXiv:2405.20039},
  year={2024}
}

@article{gan2024prediction,
  title={Prediction de-correlated inference: A safe approach for post-prediction inference},
  author={Gan, Feng and Liang, Wanfeng and Zou, Changliang},
  journal={Australian \& New Zealand Journal of Statistics},
  volume={66},
  number={4},
  pages={417--440},
  year={2024},
  publisher={Wiley Online Library}
}

@article{mccaw2023leveraging,
  title={Leveraging a surrogate outcome to improve inference on a partially missing target outcome},
  author={McCaw, Zachary R and Gaynor, Sheila M and Sun, Ryan and Lin, Xihong},
  journal={Biometrics},
  volume={79},
  number={2},
  pages={1472--1484},
  year={2023},
  publisher={Wiley Online Library}
}

@article{wang2020methods,
  title={Methods for correcting inference based on outcomes predicted by machine learning},
  author={Wang, Siruo and McCormick, Tyler H and Leek, Jeffrey T},
  journal={Proceedings of the National Academy of Sciences},
  volume={117},
  number={48},
  pages={30266--30275},
  year={2020},
  publisher={National Academy of Sciences}
}

@article{gronsbell2024another,
  title={Another look at inference after prediction},
  author={Gronsbell, Jessica and Gao, Jianhui and Shi, Yaqi and McCaw, Zachary R and Cheng, David},
  journal={arXiv preprint arXiv:2411.19908},
  year={2024}
}

@article{song2024general,
  title={A general {M}-estimation theory in semi-supervised framework},
  author={Song, Shanshan and Lin, Yuanyuan and Zhou, Yong},
  journal={Journal of the American Statistical Association},
  volume={119},
  number={546},
  pages={1065--1075},
  year={2024},
  publisher={Taylor \& Francis}
}

@article{ding2021retiring,
  title={Retiring adult: New datasets for fair machine learning},
  author={Ding, Frances and Hardt, Moritz and Miller, John and Schmidt, Ludwig},
  journal={Advances in neural information processing systems},
  volume={34},
  pages={6478--6490},
  year={2021}
}

@inproceedings{tian2023just,
  title={Just Ask for Calibration: Strategies for Eliciting Calibrated Confidence Scores from Language Models Fine-Tuned with Human Feedback},
  author={Tian, Katherine and Mitchell, Eric and Zhou, Allan and Sharma, Archit and Rafailov, Rafael and Yao, Huaxiu and Finn, Chelsea and Manning, Christopher D},
  booktitle={Proceedings of the 2023 Conference on Empirical Methods in Natural Language Processing},
  pages={5433--5442},
  year={2023}
}

@inproceedings{baly2020we,
  title={We Can Detect Your Bias: Predicting the Political Ideology of News Articles},
  author={Baly, Ramy and Da San Martino, Giovanni and Glass, James and Nakov, Preslav},
  booktitle={Proceedings of the 2020 Conference on Empirical Methods in Natural Language Processing (EMNLP)},
  pages={4982--4991},
  year={2020}
}

@inproceedings{danescu2013computational,
  title={A Computational Approach to Politeness with Application to Social Factors},
  author={Danescu-Niculescu-Mizil, Cristian and Sudhof, Moritz and Jurafsky, Dan and Leskovec, Jure and Potts, Christopher},
  booktitle={51st Annual Meeting of the Association for Computational Linguistics},
  pages={250--259},
  year={2013},
  organization={ACL}
}

@inproceedings{li2023coannotating,
  title={CoAnnotating: Uncertainty-Guided Work Allocation between Human and Large Language Models for Data Annotation},
  author={Li, Minzhi and Shi, Taiwei and Ziems, Caleb and Kan, Min-Yen and Chen, Nancy and Liu, Zhengyuan and Yang, Diyi},
  booktitle={Proceedings of the 2023 Conference on Empirical Methods in Natural Language Processing},
  pages={1487--1505},
  year={2023}
}

@article{robins1994estimation,
  title={Estimation of regression coefficients when some regressors are not always observed},
  author={Robins, James M and Rotnitzky, Andrea and Zhao, Lue Ping},
  journal={Journal of the American statistical Association},
  volume={89},
  number={427},
  pages={846--866},
  year={1994},
  publisher={Taylor \& Francis}
}

@article{robins1995semiparametric,
  title={Semiparametric efficiency in multivariate regression models with missing data},
  author={Robins, James M and Rotnitzky, Andrea},
  journal={Journal of the American Statistical Association},
  volume={90},
  number={429},
  pages={122--129},
  year={1995},
  publisher={Taylor \& Francis}
}

@incollection{rubin2018multiple,
  title={Multiple imputation},
  author={Rubin, Donald B},
  booktitle={Flexible imputation of missing data, second edition},
  pages={29--62},
  year={2018},
  publisher={Chapman and Hall/CRC}
}

@book{tsiatis2006semiparametric,
  title={Semiparametric theory and missing data},
  author={Tsiatis, Anastasios A},
  volume={4},
  year={2006},
  publisher={Springer}
}

@article{glynn2010introduction,
  title={An introduction to the augmented inverse propensity weighted estimator},
  author={Glynn, Adam N and Quinn, Kevin M},
  journal={Political analysis},
  volume={18},
  number={1},
  pages={36--56},
  year={2010},
  publisher={Cambridge University Press}
}

@article{settles2009active,
  title={Active learning literature survey},
  author={Settles, Burr},
  year={2009},
  publisher={University of Wisconsin-Madison Department of Computer Sciences}
}

@article{hahn2011adaptive,
  title={Adaptive experimental design using the propensity score},
  author={Hahn, Jinyong and Hirano, Keisuke and Karlan, Dean},
  journal={Journal of Business \& Economic Statistics},
  volume={29},
  number={1},
  pages={96--108},
  year={2011},
  publisher={Taylor \& Francis}
}

@inproceedings{roy2001toward,
  title={Toward optimal active learning through sampling estimation of error reduction},
  author={Roy, Nicholas and McCallum, Andrew},
  booktitle={ICML},
  volume={1},
  number={3},
  pages={5},
  year={2001},
  organization={Citeseer}
}

@book{fedorov2013theory,
  title={Theory of optimal experiments},
  author={Fedorov, Valerii Vadimovich},
  year={2013},
  publisher={Elsevier}
}

@inproceedings{kossen2021active,
  title={Active testing: Sample-efficient model evaluation},
  author={Kossen, Jannik and Farquhar, Sebastian and Gal, Yarin and Rainforth, Tom},
  booktitle={International Conference on Machine Learning},
  pages={5753--5763},
  year={2021},
  organization={PMLR}
}
\bibliographystyle{plainnat}

\newpage
\appendix

\section{Proofs}\label{appendix:proofs}

\subsection{Proof of Theorem \ref{thm:path}}
By the definition of $\hat{\theta}^{\pi^{(\rho)}}$, we have 
$$
\hat{\theta}^{\pi^{(\rho)}}=\frac{1}{n} \sum_{i=1}^n\left(f\left(X_i\right)+\left(Y_i-f\left(X_i\right)\right) \frac{\xi_i}{\pi^{(\rho)}\left(X_i\right)}\right).
$$
From the assumption, we have $\hat{\rho} = \rho^* + o_P(1)$. 
By the continuity of the budget-preserving path $\pi^{(\rho)}$, it follows that $\pi^{(\hat{\rho})}(X_i) = \pi^{(\rho^*)}(X_i) + o_P(1)$ for any $i\in\{1, \ldots, n\}$. This, as a result, gives $\hat{\theta}^{\pi^{(\hat\rho)}} = \hat{\theta}^{\pi^{(\rho^*)}} + o_P(1)$ by the continuity of $\hat{\theta}^{\pi^{(\rho)}}$.

It follows from Proposition 1 in \cite{zrnic2024active} that we have 
$$
\sqrt{n}\left(\hat{\theta}^{\pi^{(\rho^*)}} -\theta^*\right) \xrightarrow{d} \mathcal{N}\left(0, \sigma_{\rho^*}^2\right),
\quad \sigma_{\rho^*}^2 = \Var (\hat{\theta}^{\pi^{(\rho^*)}}).$$
Since $\hat{\theta}^{\pi^{(\hat\rho)}} \xrightarrow{p} \hat{\theta}^{\pi^{(\rho^*)}}$, 
$$
\sqrt{n}\left(\hat{\theta}^{\pi^{(\hat{\rho})}} -\theta^*\right) \xrightarrow{d} \mathcal{N}\left(0, \sigma_{\rho^*}^2\right).
$$

By the definition of $\rho^*$, $\rho^* = \mathop{\arg \min}\limits_{\rho} \Var (\hat{\theta}^{\pi^{(\rho)}}) $, we have 
$$
\sigma_{\rho^*}^2 = \Var (\hat{\theta}^{\pi^{(\rho^*)}}) \leq \min\{\Var (\hat{\theta}^{\pi^{(0)}}),\Var (\hat{\theta}^{\pi^{(1)}})\} = \min\{\sigma_0^2, \sigma_1^2\}.
$$

This completes the proof.

\subsection{A sufficient condition for $\hat{\rho} = \rho^* + o_P(1)$}

\begin{Proposition}
    Suppose $\hat{e}^2(X) = e^2(X) + o_P(1)$, and $\rho^*$ is unique. Suppose $\hat{e}(X)$ is uniformly upper bounded by $M>0$. Suppose further that $\pi^{(\rho)}(X)$ is uniformly lower-bounded by $m > 0$, then we have $\hat{\rho} = \rho^* + o_P(1)$.
\end{Proposition}
\begin{proof}
    Denote 
    \[
    \mathcal{F} = \left\{f_{\rho}(x) = \frac{\hat{e}^2(x)}{\pi^{(\rho)}(x)}: \rho \in [0,1]\right\}.
    \]
    
    We first show that $\mathcal{F}$ is a P-Glivenko-Cantelli class.

    Since $\pi^{(\rho)}$ is continuous, and supported on $[0,1]$, it is  uniformly continuous on $[0,1]$. Hence for any $\delta > 0$, there exists $\eta>0$ such that $|\pi^{(\rho_1)}(X) - \pi^{(\rho_2)}(X)|\leq \frac{m^2}{M^2}\delta$ whenever $|\rho_1 - \rho_2| \leq \eta$. Now, we cover $[0,1]$ with a grid $0=\rho_0 < \rho_1 < \dots < \rho_K = 1$, where $\rho_k - \rho_{k-1} = \eta$ for $k\leq K-1$. Then, for any $\rho \in [\rho_{k-1}, \rho_{k}]$, we have 
    \[
    |f_{\rho}(x) - f_{\rho_{k-1}}(x)| = \left|\frac{\hat{e}^2(x)}{\pi^{(\rho)}(x)} - \frac{\hat{e}^2(x)}{\pi^{(\rho_{k-1})}(x)}\right| \leq \frac{M^2}{m^2}\left|\pi^{(\rho)}(x) - \pi^{(\rho_{k-1})}(x)\right| \leq \delta.
    \]
    Hence $[f_{\rho_{k-1}} - \delta, f_{\rho_{k-1}} + \delta]$ is an $2\delta$-bracket in $L_1(P)$ that contains every $f_{\rho}$ with $\rho \in [\rho_{k-1}, \rho_k]$. So the bracketing number $N_{[]}$ is finite, $N_{[]}(2\delta, \mathcal{F}, L_1(P)) \leq K \leq \frac{1}{\eta} + 1  < \infty.$ We thus conclude from the Blum-DeHardt theorem that $\mathcal{F}$ is a P-Glivenko-Cantelli class. Consequently, we have 
    \[
    \sup_{\rho \in [0,1]} \left|\frac{1}{n}\sum_{i=1}^{n}\frac{\hat{e}^2(X_i)}{\pi^{(\rho)}(X_i)} - \EE \frac{\hat{e}^2(X)}{\pi^{(\rho)}(X)}\right| \xrightarrow{p} 0.
    \]
This implies that 
\[
\left|\inf_{\rho \in[0,1] }\frac{1}{n}\sum_{i=1}^{n}\frac{\hat{e}^2(X_i)}{\pi^{(\rho)}(X_i)} - \inf_{\rho \in[0,1] }\EE \frac{\hat{e}^2(X)}{\pi^{(\rho)}(X)}\right| \xrightarrow{p} 0.
\]

By definition, $\hat{\rho} = \mathop{\arg\min}\limits_{\rho} \frac{1}{n}\sum_{i=1}^{n}\frac{\hat{e}^2(X_i)}{\pi^{(\rho)}(X_i)}$. Denote $S = \mathop{\arg\min}\limits_{\rho} \EE \frac{\hat{e}^2(X)}{\pi^{(\rho)}(X)}$. Then, by continuity of $\pi^{(\rho)}(X)$, we have $d(\hat{\rho}, S) \xrightarrow{p} 0$, for $d(\hat{\rho}, S) = \inf\{|\hat{\rho} - \hat{\rho}^*|: \hat{\rho}^* \in S\}$.

Now, for any $\hat{\rho}^* \in S$, we have 
\[
\begin{aligned}
    \EE \frac{e^2(X)}{\pi^{(\hat{\rho}^*)}(X)} \leq~&\EE\frac{\hat{e}^2(X) +o_P(1)}{\pi^{(\hat{\rho}^*)}(X)}\\
    \leq~& \EE\frac{\hat{e}^2(X)}{\pi^{(\rho^*)}(X)} +  o_P(1)\EE\frac{1}{\pi^{(\hat{\rho}^*)}(X)}\\
    \leq~& \EE\frac{e^2(X) +  o_P(1)}{\pi^{(\rho^*)}(X)} + o_P(1)\EE\frac{1}{\pi^{(\hat{\rho}^*)}(X)}\\
    =~& \EE\frac{e^2(X)}{\pi^{(\rho^*)}(X)} + o_P(1)\EE\left[\frac{1}{\pi^{(\rho^*)}(X)} + \frac{1}{\pi^{(\hat{\rho}^*)}(X)}\right]\\
    =~& \EE\frac{e^2(X)}{\pi^{(\rho^*)}(X)} + o_P(1).
\end{aligned}
\]
Since
\[
\operatorname{Var} (\hat{\theta}^{\pi^{(\rho)}}) = \mathbb{E} \left(\frac{e^2(X)}{\pi^{(\rho)}(X)}\right) + C,
\]
where \( C \) is a constant independent of \( \rho \), we have 
\[\operatorname{Var} (\hat{\theta}^{\pi^{(\hat{\rho}^*)}}) \leq \operatorname{Var} (\hat{\theta}^{\pi^{(\rho^*)}}) + o_P(1).
\]
On the other hand, by the definition of $\rho^*$,  
\[
\operatorname{Var} (\hat{\theta}^{\pi^{(\hat{\rho}^*)}}) \geq \operatorname{Var} (\hat{\theta}^{\pi^{(\rho^*)}}) 
\]
also holds. Whence $\operatorname{Var} (\hat{\theta}^{\pi^{(\hat{\rho}^*)}}) \xrightarrow{p} \operatorname{Var} (\hat{\theta}^{\pi^{(\rho^*)}})$.

Since \( \rho^* \) is the unique minimizer of \( \operatorname{Var} (\hat{\theta}^{\pi^{(\rho)}}) \), $
\hat{\rho}^* \xrightarrow{p} \rho^*$ 
by continuity. 
Since $d(\hat{\rho}, S) \xrightarrow{p} 0$ and $\hat{\rho}^*$ is an arbitrary element in $S$, we immediately conclude that 

\[
\hat{\rho} \xrightarrow{p} \rho^*.
\]

\end{proof}

\subsection{Proof of Theorem \ref{thm:path_m}}
By the definition of $\hat{\theta}^{\pi^{(\rho)}}$, we have 
$$
\hat{\theta}^{\pi^{(\rho)}}=\underset{\theta}{\arg \min }\frac{1}{n} \sum_{i=1}^n\left(\ell_{\theta, i}^f+\left(\ell_{\theta, i}-\ell_{\theta, i}^f\right) \frac{\xi_i}{\pi^{(\rho)}(X_i)}\right).
$$
We assume $\hat{\rho} = \rho^* + o_P(1)$. 
By the continuity of the budget-preserving path $\pi^{(\rho)}$, it follows that $\pi^{(\hat{\rho})}(X_i) = \pi^{(\rho^*)}(X_i) + o_P(1)$ for any $i\in\{1, \ldots, n\}$. This, as a result, gives $\hat{\theta}^{\pi^{(\hat\rho)}} = \hat{\theta}^{\pi^{(\rho^*)}} + o_P(1)$ by the continuity of $\ell_{\theta, i}^f+\left(\ell_{\theta, i}-\ell_{\theta, i}^f\right) \frac{\xi_i}{\pi^{(\rho)}(X_i)}$ with respect to $\theta$.

Given the assumption that $\hat{\theta}^{\pi^{(\rho^*)}} \xrightarrow{p} \theta^*$, from Theorem 1 in \cite{zrnic2024active}, we have 
$$
\sqrt{n}\left(\hat{\theta}^{\pi^{(\rho^*)}} -\theta^*\right) \xrightarrow{d} \mathcal{N}\left(0, \Sigma_{\rho^*}\right),
$$

where $\Sigma_{\rho^*} = H_{\theta^*}^{-1} \operatorname{Var}\left(\nabla \ell_{\theta^*, i}^f+\left(\nabla \ell_{\theta^*, i}-\nabla \ell_{\theta^*, i}^f\right) \frac{\xi}{\pi^{(\rho^*)}(X_i)}\right) H_{\theta^*}^{-1}$.

Since $\hat{\theta}^{\pi^{(\hat\rho)}} \xrightarrow{p} \hat{\theta}^{\pi^{(\rho^*)}}$, 
$$
\sqrt{n}\left(\hat{\theta}^{\pi^{(\hat{\rho})}} -\theta^*\right) \xrightarrow{d} \mathcal{N}\left(0, \Sigma_{\rho^*}\right).
$$
The definition $\rho^* = \mathop{\arg \min}\limits_{\rho} \Sigma^{\pi^{(\rho)}}_{jj}$ yields  
\[
\begin{aligned}
    \Sigma_{\rho^*,jj} 
    = \Sigma^{\pi^{(\rho^*)}}_{jj} 
    \leq \min\{\Sigma^{\pi^{(0)}}_{jj},\Sigma^{\pi^{(1)}}_{jj}\} 
    = \min\{\Sigma_{0,jj}, \Sigma_{1,jj}\}.
\end{aligned}
\]
This completes the proof.

\subsection{Setting $\pi \propto \hat{e}$ leads to a trivial choice of $\hat{\rho}=0$ when not incorporating robustness constraint}

Starting from the variance estimate used in the optimization objective 
\[\hat\rho = \argmin_\rho \frac{1}{n}\sum_{i=1}^{n}\frac{\hat{e}^2(X_i)}{\pi^{(\rho)}(X_i)},\] by the Cauchy-Schwarz inequality, for any $\pi$ such that $\sum_{i=1}^n \pi\left(X_i\right)=n_b$ (i.e. satisfying the budget constraint), $\sum_{i=1}^n \frac{\hat{e}^2\left(X_i\right)}{\pi\left(X_i\right)} \geq \frac{\left(\sum_{i=1}^n \hat{e}\left(X_i\right)\right)^2}{\sum_{i=1}^n \pi\left(X_i\right)}=\frac{\left(\sum_{i=1}^n \hat{e}\left(X_i\right)\right)^2}{n_b}$. The equality holds when $\pi \propto \hat{e}$, which corresponds to $\pi^{(\rho)}$ with $\rho=0$.

\section{A natural family of budget-preserving paths}\label{appendix:geodesic}

Among the diverse set of possible paths \citep{seguy2015principal, liero2016optimal}, it is natural to consider \emph{geodesic paths}, which are a family of ``shortest paths.''

\begin{Definition}[Geodesic \citep{burago2001course}] 
A curve $\gamma: I \rightarrow M$ from an interval $I\subseteq \R$ to a metric space $M$ with metric $d$ is a geodesic if there is a constant $v \geq 0$ such that for any $\rho \in I$ there is a neighborhood $J$ of $\rho$ in $I$ such that for any $\rho_1, \rho_2 \in J$ we have

$$
d\left(\gamma\left(\rho_1\right), \gamma\left(\rho_2\right)\right)=v\left|\rho_1-\rho_2\right|.
$$

\end{Definition}

We revisit the examples from Section \ref{sec:warmup} and provide more geodesic paths.

In all the following examples, we assume $P$ and $Q$ have the same support.

\begin{example}[Linear path]
    The linear path, $\pi^{(\rho)} \propto (1-\rho)\pi + \rho \pi^{\mathrm{unif}}$, is the geodesic path with respect to $d(P, Q) = \|P-Q\|$ with $v = \|\pi - \pi^{\mathrm{unif}}\|$. Here, $\|\cdot\|$ is any norm.
\end{example}

\begin{example}[Geometric path] The geometric path, $\pi^{(\rho)} \propto \pi^{1-\rho}(\pi^{\mathrm{unif}})^{\rho}$, is the geodesic path with respect to $d(P, Q) = \|\log P - \log Q \|$ with $v = \|\log \pi - \log \pi^{\mathrm{unif}}\|$. Here, $\log$ is taken element-wise.
\end{example}

\begin{example}[Hellinger path] The Hellinger path, $\pi^{(\rho)} \propto \left((1-\rho)\sqrt{\pi} + \rho \sqrt{\pi^{\mathrm{unif}}}\right)^2$, is the geodesic path with respect to $d(P, Q) = \|\sqrt{P} - \sqrt{Q} \|$ with $v = \|\sqrt{\pi} - \sqrt{\pi^{\mathrm{unif}}}\|$. Here, the square root is taken element-wise.
\end{example}

\textbf{Note (more examples).} Some distance metrics may not have an analytical characterization for their corresponding geodesic path, such as the Wasserstein and Jensen-Shannon distances. However, it is computationally tractable to solve for a geodesic path numerically up to a tolerance margin for many well-defined distance metrics. For example, when computing the geodesic for the Jensen-Shannon distance, we can discretize the interval \([0,1]\) into \(N\) segments so that \(P_0 = P\) and \(P_N = Q\), and we define a series of intermediate distributions \(P_1, P_2, \dots, P_{N-1}\). The task is then cast as an optimization problem: we minimize the total path length computed as the sum of the square roots of the Jensen-Shannon divergences between successive distributions, i.e., \(\sum_{i=0}^{N-1} \sqrt{\mathrm{JS}(P_i, P_{i+1})}\). Here, $\mathrm{JS}(P\|Q)=\frac{1}{2}D(P\|M) + \frac{1}{2}D(Q\|M)$, where $M=\frac{1}{2}(P+Q)$. This is a constrained optimization problem and can be solved by standard gradient-based methods.

\subsection{Uniqueness of $\rho^*$}

In Section \ref{sec:warmup}, we saw that the uniqueness of the optimal $\rho^*$ and the consistency of $\hat{e}$ are sufficient conditions for the consistency of $\hat{\rho}$. 
In the case of all three budget-preseving paths from the previous section, it can be easily verified by computing the second derivative of $\Var(\hat\theta^{\pi^{(\rho)}})$ that this variance is strictly convex and thus $\rho^*$ is unique. We include the corresponding proofs for completeness.

\paragraph{Linear path.}
We have $\pi^{(\rho)}(X) = (1-\rho) \pi(X) + \rho \frac{n_b}{n}$. The problem of minimizing $\Var(\hat\theta^{\pi^{(\rho)}})$ is equivalent to
\[ \mathop{\arg\min}\limits_{\rho} \EE \left[\frac{(Y-f(X))^2}{(1-\rho)\pi(X) + \rho \frac{n_b}{n}}\right].\]

Denoting $g(\rho) = \EE \left[\frac{(Y-f(X))^2}{(1-\rho)\pi(X) + \rho \frac{n_b}{n}}\right]$, we have 
\[g^{\prime}(\rho) = \EE \left[\frac{-\left(Y-f(X)\right)^2\left(\frac{n_b}{n} - \pi(X)\right)}{\left((1-\rho)\pi(X) + \rho \frac{n_b}{n}\right)^2}\right],\]
and 
\[g^{\prime \prime}(\rho) = \EE \left[\frac{2\left(Y-f(X)\right)^2\left(\frac{n_b}{n} - \pi(X)\right)^2}{\left((1-\rho)\pi(X) + \rho \frac{n_b}{n}\right)^3}\right].\]

Clearly, $g^{\prime \prime}(\rho) > 0$, which means that $g(\rho)$ is convex. Hence, there is a unique optimal value of $\rho$ in $[0,1]$. 

Notice that $g^{\prime}(1) = \frac{n^2}{n_b^2} \EE \left[ (Y-f(X))^2\left(\pi(X) - \frac{n_b}{n}\right)\right]$. Hence, if $\EE \left[ (Y-f(X))^2\pi(X)\right] > \frac{n_b}{n}\EE\left[ (Y-f(X))^2\right]$, then $g^{\prime}(1) > 0$, which implies that the optimal $\rho$ lies in $[0,1)$.

\paragraph{Geometric path.}
Consider the path $\pi^{(\rho)}(X) \propto \pi(X)^{1-\rho}(\pi^{\mathrm{unif}})^{\rho}$; in particular, $\pi^{(\rho)}(X) = \frac{n_b}{n} \frac{\pi(X)^{1-\rho}}{\EE[\pi(X)^{1-\rho}]}$. 

Similar to the last example, we denote $g(\rho) = \EE\left[\frac{(Y-f(X))^2}{\pi^{(\rho)}(X)}\right] = \frac{n}{n_b}\EE\left[\frac{(Y-f(X))^2}{\pi(X)^{1-\rho}}\right]\EE\left[\pi(X)^{1-\rho}\right]$. Then, we have

$$
    g'(\rho) = \frac{n}{n_b} \EE\left[\frac{(Y-f(X))^2}{\pi(X)^{1-\rho}}\log \pi(X)\right]\EE\left[\pi(X)^{1-\rho}\right] -\frac{n}{n_b}\EE\left[\frac{(Y-f(X))^2}{\pi(X)^{1-\rho}}\right]\EE\left[\pi(X)^{1-\rho} \log \pi(X)\right],
$$

and 

$$
\begin{aligned}
    g''(\rho) = &\frac{n}{n_b} \EE\left[\frac{(Y-f(X))^2}{\pi(X)^{1-\rho}}\log ^2\pi(X)\right]\EE\left[\pi(X)^{1-\rho}\right] +\frac{n}{n_b}\EE\left[\frac{(Y-f(X))^2}{\pi(X)^{1-\rho}}\right]\EE\left[\pi(X)^{1-\rho} \log^2 \pi(X)\right]\\
    &- 2 \frac{n}{n_b} \EE\left[\frac{(Y-f(X))^2}{\pi(X)^{1-\rho}} \log \pi(X)\right]\EE\left[\pi(X)^{1-\rho} \log \pi(X)\right].
\end{aligned}
$$

Since $(Y-f(X))^2 \geq 0$, $\pi(X) > 0$, and $\log^2 \pi(X) \geq 0$, we have that 

$$
\begin{aligned}
    &\EE\left[\frac{(Y-f(X))^2}{\pi(X)^{1-\rho}}\log ^2\pi(X)\right]\EE\left[\pi(X)^{1-\rho}\right] +\EE\left[\frac{(Y-f(X))^2}{\pi(X)^{1-\rho}}\right]\EE\left[\pi(X)^{1-\rho} \log^2 \pi(X)\right]\\
    \geq & 2\sqrt{\EE\left[\frac{(Y-f(X))^2}{\pi(X)^{1-\rho}}\log ^2\pi(X)\right]\EE\left[\pi(X)^{1-\rho}\right]\EE\left[\frac{(Y-f(X))^2}{\pi(X)^{1-\rho}}\right]\EE\left[\pi(X)^{1-\rho} \log^2 \pi(X)\right]}\\
    = & 2\sqrt{\EE\left[\frac{(Y-f(X))^2}{\pi(X)^{1-\rho}}\log ^2\pi(X)\right]\EE\left[\frac{(Y-f(X))^2}{\pi(X)^{1-\rho}}\right]\EE\left[\pi(X)^{1-\rho}\right]\EE\left[\pi(X)^{1-\rho} \log^2 \pi(X)\right]}\\
    \geq & 2\sqrt{ \EE^2\left[\frac{(Y-f(X))^2}{\pi(X)^{1-\rho}} \log \pi(X)\right]\EE^2\left[\pi(X)^{1-\rho} \log \pi(X)\right]}\\
    = & \EE\left[\frac{(Y-f(X))^2}{\pi(X)^{1-\rho}} \log \pi(X)\right]\EE\left[\pi(X)^{1-\rho} \log \pi(X)\right].
\end{aligned}
$$

The last inequality follows from the Cauchy-Schwarz inequality. Therefore, we have $g''(\rho) \geq 0$. Further, if $\pi(X) \neq \pi^{\mathrm{unif}}$, the inequality is strict, which means $g(\rho)$ is convex. Thus, there is a unique optimal value of $\rho$ in $[0,1]$.

\paragraph{Hellinger path.}
Suppose $P$ and $Q$ are two discrete distributions. The Hellinger distance between $P$ and $Q$ is $H(P, Q) = \frac{1}{\sqrt{2}} \|\sqrt{P} - \sqrt{Q} \|_2$. The geodesic connecting $\pi(X)$ and $\pi^{\mathrm{unif}} = \frac{n_b}{n}$ is:

$$
\pi^{(\rho)}(X) = \left(\frac{\sin ((1-\rho)\beta)}{\sin \beta}\sqrt{\pi(X)} +  \frac{\sin (\rho\beta)}{\sin \beta}\sqrt{\frac{n_b}{n}}\right)^2,
$$

where $\beta = \arccos \left(\sum_{i=1}^{n} \sqrt{\frac{\pi(X_i)}{n}\cdot n_b}\right)$.

Similarly as above, minimizing the variance $\Var(\hat\theta^{\pi^{(\rho)}})$ amounts to minimizing the function $$g(\rho) = \EE \left[\frac{(Y-f(X))^2}{\left(\frac{\sin ((1-\rho)\beta)}{\sin \beta}\sqrt{\pi(X)} +  \frac{\sin (\rho\beta)}{\sin \beta}\sqrt{\frac{n_b}{n}}\right)^2}\right]$$ over $\rho$. The derivative $g^{\prime}(\rho)$ is given by 
\[
-2 \EE \left[(Y-f(X))^2\left(\frac{\sin ((1-\rho)\beta)}{\sin \beta}\sqrt{\pi(X)} +  \frac{\sin (\rho\beta)}{\sin \beta}\sqrt{\frac{n_b}{n}}\right)^{-3}\left(-\beta \frac{\cos ((1-\rho)\beta)}{\sin \beta}\sqrt{\pi(X)} +  \beta\frac{\cos (\rho\beta)}{\sin \beta}\sqrt{\frac{n_b}{n}}\right)\right],
\] 
while the second $g^{\prime\prime}(\rho)$ is given by 
\begin{align*}
\EE &\left[(Y-f(X))^2\left(\frac{\sin ((1-\rho)\beta)}{\sin \beta}\sqrt{\pi(X)} +  \frac{\sin (\rho\beta)}{\sin \beta}\sqrt{\frac{n_b}{n}}\right)^{-4}  
\left[\vphantom{\sum_{i=1}^m} 6\left(-\beta \frac{\cos ((1-\rho)\beta)}{\sin \beta}\sqrt{\pi(X)} \right.\right.\right.\\
& \left.\left.\left. + \beta\frac{\cos (\rho\beta)}{\sin \beta}\sqrt{\frac{n_b}{n}}\right)^2 + 2\beta^2 \left(\frac{\sin ((1-\rho)\beta)}{\sin \beta}\sqrt{\pi(X)} +  \frac{\sin (\rho\beta)}{\sin \beta}\sqrt{\frac{n_b}{n}}\right)^{2}\right]\vphantom{\sum_{i=1}^m}\right] > 0.
\end{align*}
Therefore, $g(\rho)$ is strictly convex, and there is a unique optimal value of $\rho$ in $[0,1]$.

\section{Perturbed model errors after robust optimization}\label{appendix:constraint}

It is natural to choose the constraint $\mathcal C$ by upper-bounding the norm of $\boldsymbol{\epsilon}$. Our default choice is the $\ell_2$ norm, i.e.~$\|\boldsymbol{\epsilon}\|_2 \leq c$. The $\ell_2$ norm can be roughly thought of as controlling the variance of the errors in $\hat e^2$. In particular, imagine $\hat e^2(X_i)$ can be viewed as a noisy version of $e^2(X_i)$: $\hat{e}^2(X_i) = e^2(X_i) + \xi_i$, where the $(X_i,\xi_i)$ pairs are i.i.d. and $\xi_i$ have mean zero. Then, by concentration, $\|\boldsymbol{\epsilon}\|_2^2 \approx \sum_{i} \Var(\xi_i)$.

In Figure \ref{fig:err_perturb} we illustrate how robust optimization over the $\ell_2$ set $\mathcal C$ recovers errors $\hat e^2(X_i) + \epsilon_i$ that are much closer to $e^2(X_i)$ than simply using $\hat e^2(X_i)$.

\begin{figure}[!htbp]
    \centering
    \includegraphics[width=0.75\linewidth]{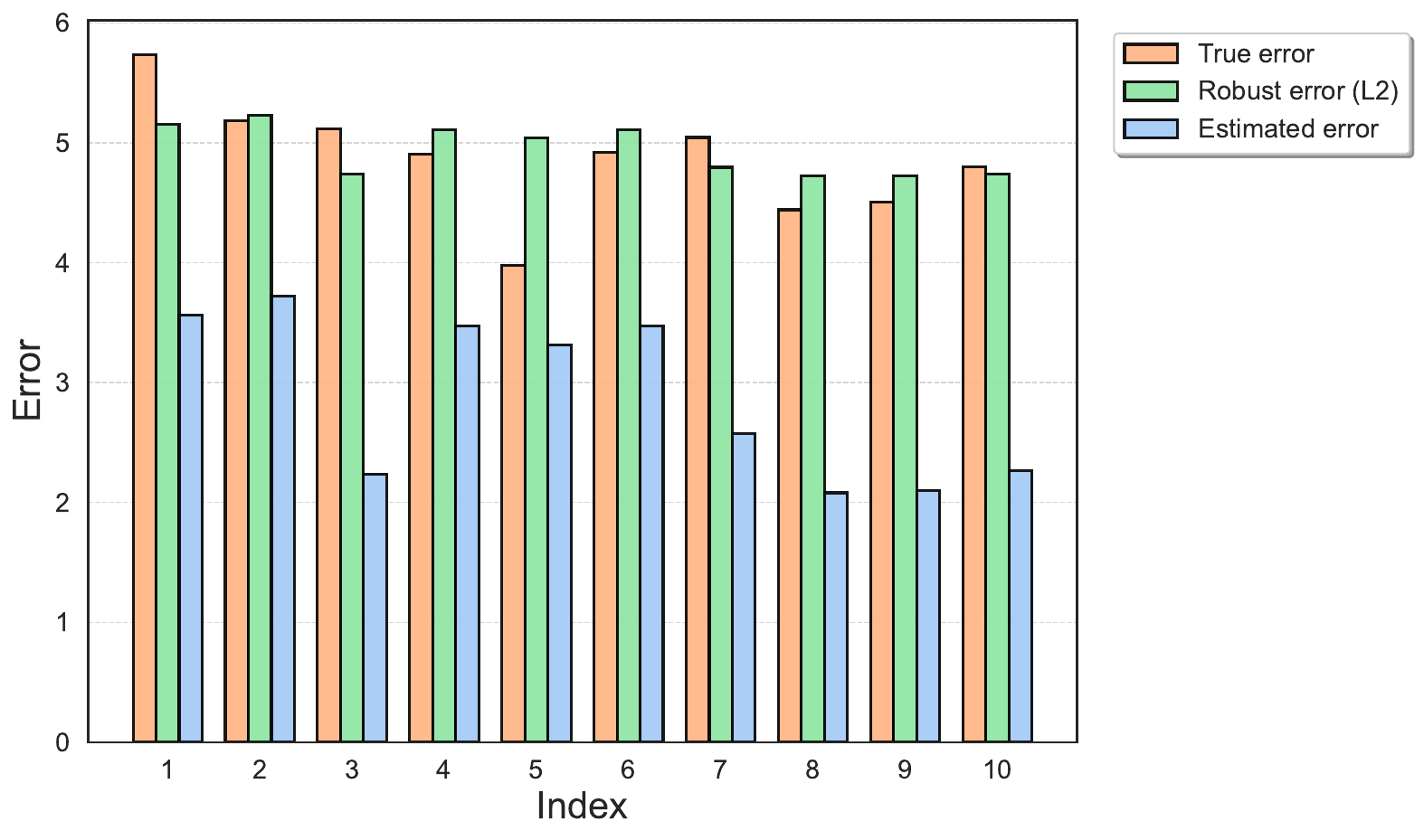}
    \caption{\textbf{Perturbed errors $\hat e^2(X_i) + \epsilon_i$ vs naive errors $\hat e^2(X_i)$ with $\ell_2$ constraint $\mathcal C$.} We consider a regime where we underestimate the true error (for example, due to the model being overconfident). We let $e(X_i) \sim \mathcal{N}(5, 0.25)$ and $\hat{e}(X_i) \sim \mathcal{N}(3, 0.25)$, and $\pi^{(\rho)}$ is the linear path with $\rho = 0.5$. The robustness constraint is $\mathcal{C} = \{\boldsymbol{\epsilon}: \|\boldsymbol{\epsilon}\|_2 \leq 50\}$. Each index $i$ corresponds to one sample $X_i$. The robust error (green bar) is the error after perturbation, $\hat e^2(X_i) + \epsilon_i$, and the estimated error (blue bar) is the error before perturbation, $\hat e^2(X_i)$. The robust errors are much closer to the estimated errors.}
    \label{fig:err_perturb}
\end{figure}

\section{A toy example: choice of $\mathcal{C}$}
A simple $ \ell_2 $ norm constraint may not always be the most powerful choice of $ C $. Zooming out, our method can in principle be combined with \emph{any} choice of $ C $, including one where we learn regions of the space where scores are systematically overconfident or underconfident. At a high level, our method (1) learns $ C $ (in our experiment, the ``learning'' is a simple fitting of $ c $ through cross-validation), and (2) solves a robust optimization problem with $ C $ in place. Your suggestion is an interesting choice of step (1).

We developed a dataset featuring a central ``hard'' region ($ |X| \le 2 $) flanked by two ``easy'' regions ($ 2 < |X| < 5 $). In the easy regions, error data was sampled from $ \mathcal{N}(2, 0.05) $. In the hard region, error was drawn from $ \mathcal{N}(1, 0.25) $. The estimator of error, $ \hat{\epsilon}(X) $, is designed to underestimate the error in the hard region and overestimate the error in the easy region. Specifically, $ \hat{\epsilon}(X) = 0.5 $ for $ |X| \le 2 $, and $ \hat{\epsilon}(X) = 2.5 $ otherwise.

Subsequently, we trained a meta-classifier, a gradient boost classifier, $ h(X) $, to identify these regions solely based on the performance of $ \hat{\epsilon}(X) $, without prior knowledge of the region boundaries.

This approach proved highly effective, with the meta-classifier achieving over 99\% accuracy in identifying the regions. This demonstrates our success in learning the error regions and enables us to separate the constraint set $ C $ based on these distinctions. For instance, $ C $ can be defined as $ \| \epsilon_{\text{easy}} \|_2 \le c_{\text{easy}} $ for the easy region ($ 2 < |X| < 5 $) and $ \| \epsilon_{\text{hard}} \|_2 \le c_{\text{hard}} $ for the hard region ($ |X| \le 2 $). Or even simpler, we can only optimize over hard regions, i.e. $ c_{\text{easy}} = 0 $. While these regions' dimensions are not fixed and depend on $ X $, this presents no practical difficulties because we have complete information about $ X $.

Next, we compared this structured constraint with the global constraint. Here, for the structured constraint, we only optimize over the hard region. The following table shows the result when $ n = 7000 $, $ n_h = 1400 $, and $ \pi \propto \hat{\epsilon} $.

\begin{table}[h!]
\centering
\begin{tabular}{lrr}
\hline
\textbf{Method} & \textbf{ESS} & \textbf{ESS Gain (\%)} \\
\hline
Uniform & 1400 & 0.00\% \\
Active & 1213 & -13.3\% \\
Robust active (global) & 1491 & 6.5\% \\
Robust active (structured) & 1495 & 6.8\% \\
\hline
\end{tabular}
\end{table}

We found that incorporating the structured constraint provided a slight gain in ESS over the global constraint while reducing the constraint size ($ c_{\text{global}} = 85 $ vs. $ c_{\text{hard}} = 75 $). This suggests that a more focused perturbation can be beneficial when we have strong knowledge of confident regions. However, we note that the global constraint remains a simple and practical approach given the limited gain.

\section{Additional experimental results}\label{appendix:more_exp}

\subsection{Plots with coverage and standard deviation}\label{appendix:plots}
In this subsection, we provide figures corresponding to the figures in the main text, where in addition to the effective sample size we also plot coverage.

\begin{figure}[!htbp]
    \centering
    \includegraphics[width=0.9\linewidth]{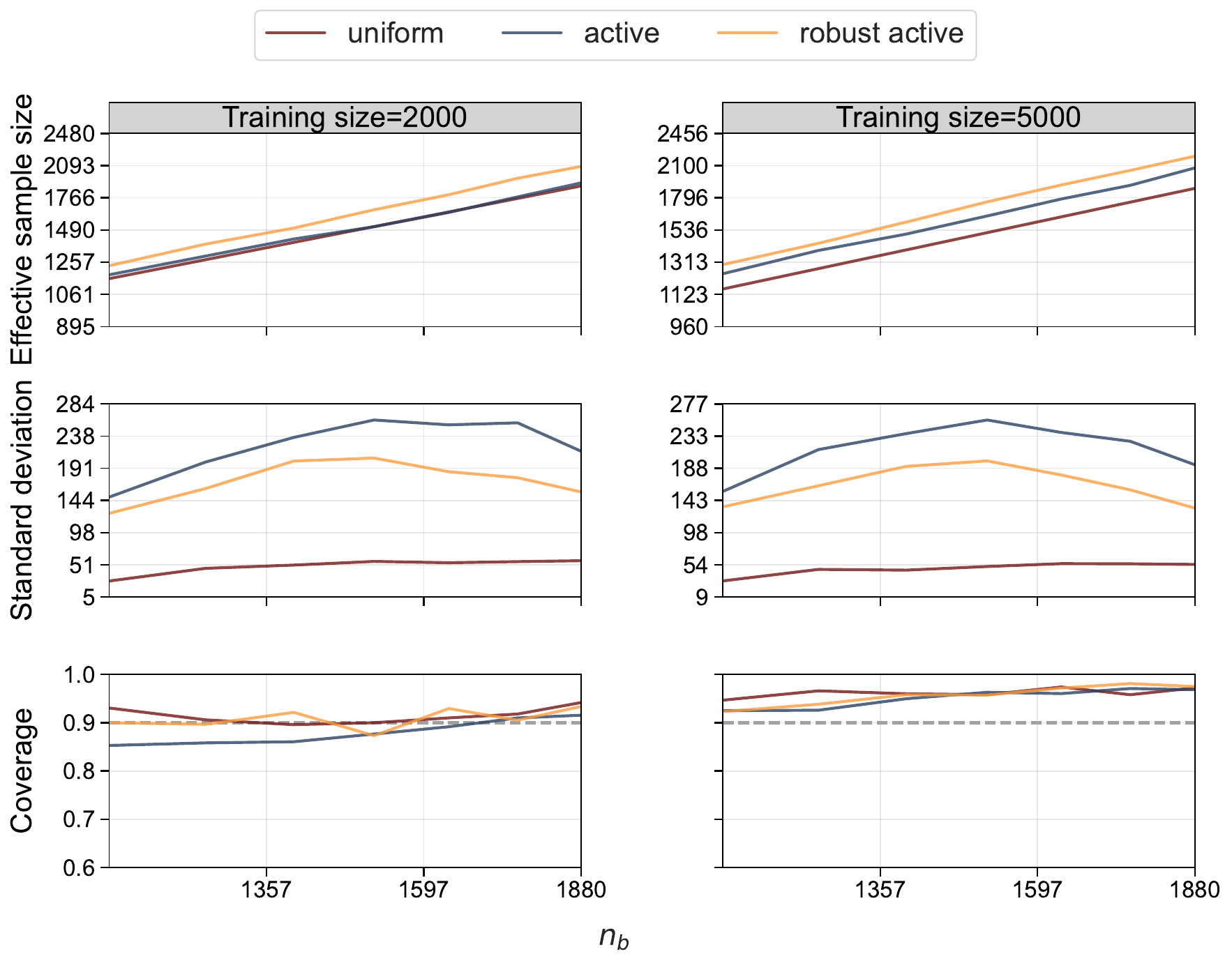}
    \caption{\textbf{Effective sample size and coverage on Pew post-election survey data}, for different dataset sizes used to train $f$. We compare uniform, active, and robust active sampling, for different values of the sampling budget $n_b$. The target of inference is the approval rate of a presidential candidate. We show the mean and one standard deviation of the effective sample size estimated over $500$ trials; in each trial we independently sample the observed labels.}
    \label{fig:election_without_robust_coverage}
\end{figure}

\begin{figure}[!htbp]
    \centering
    \includegraphics[width=0.9\linewidth]{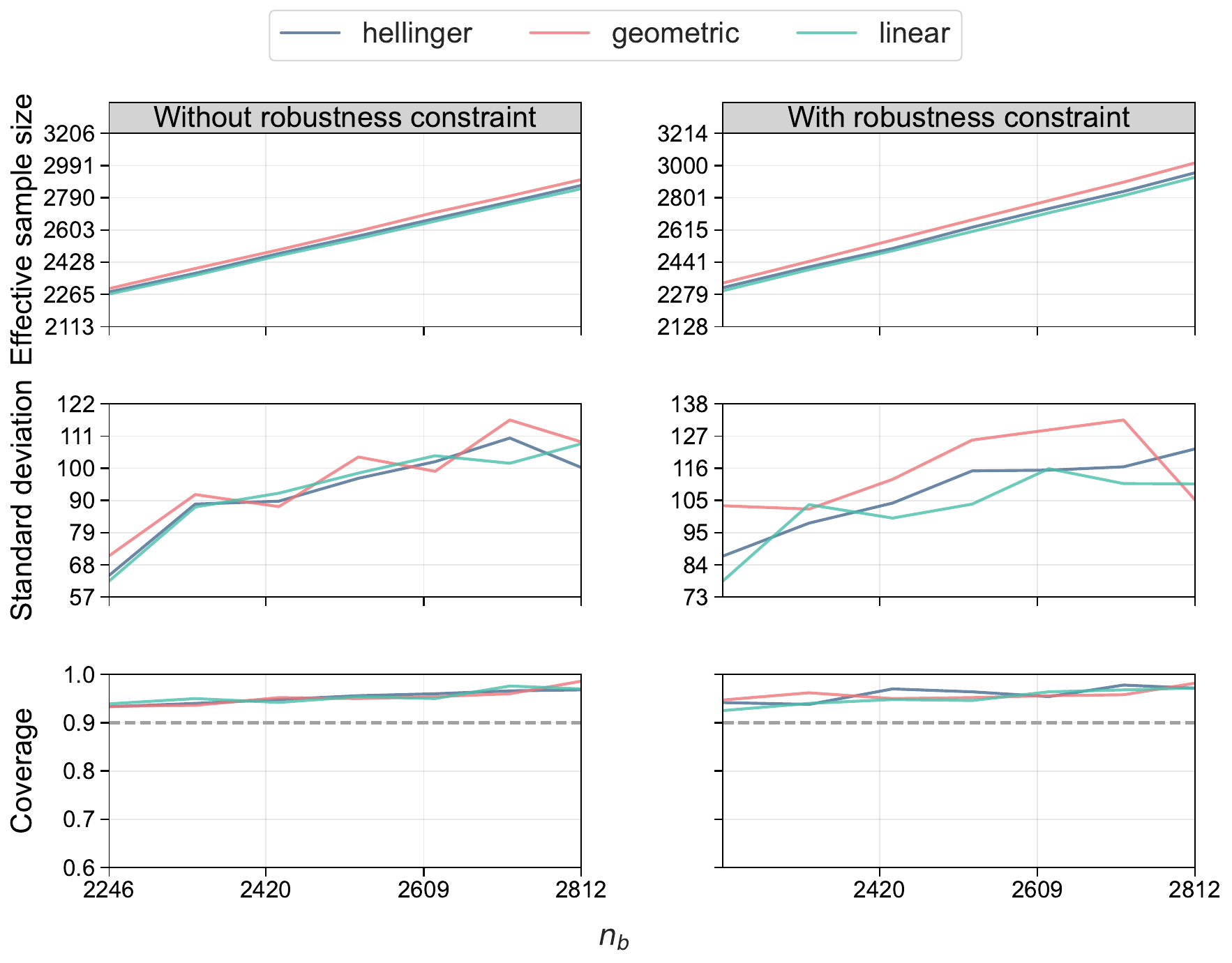}
    \caption{\textbf{Effective sample size and coverage for different budget-preserving paths on Pew post-election survey data}, without (left) and with (right) a robustness constraint $\mathcal C$. In both cases, the geometric path leads to the largest effective sample size. The target of inference is the same as in Figure \ref{fig:election_robust}. We show the mean and one standard deviation of the effective sample size estimated over $500$ trials; in each trial we independently sample the observed labels.}
    \label{fig:guide_ess_coverage}
\end{figure}

\begin{figure}[H]
    \centering
    \includegraphics[width=1\linewidth]{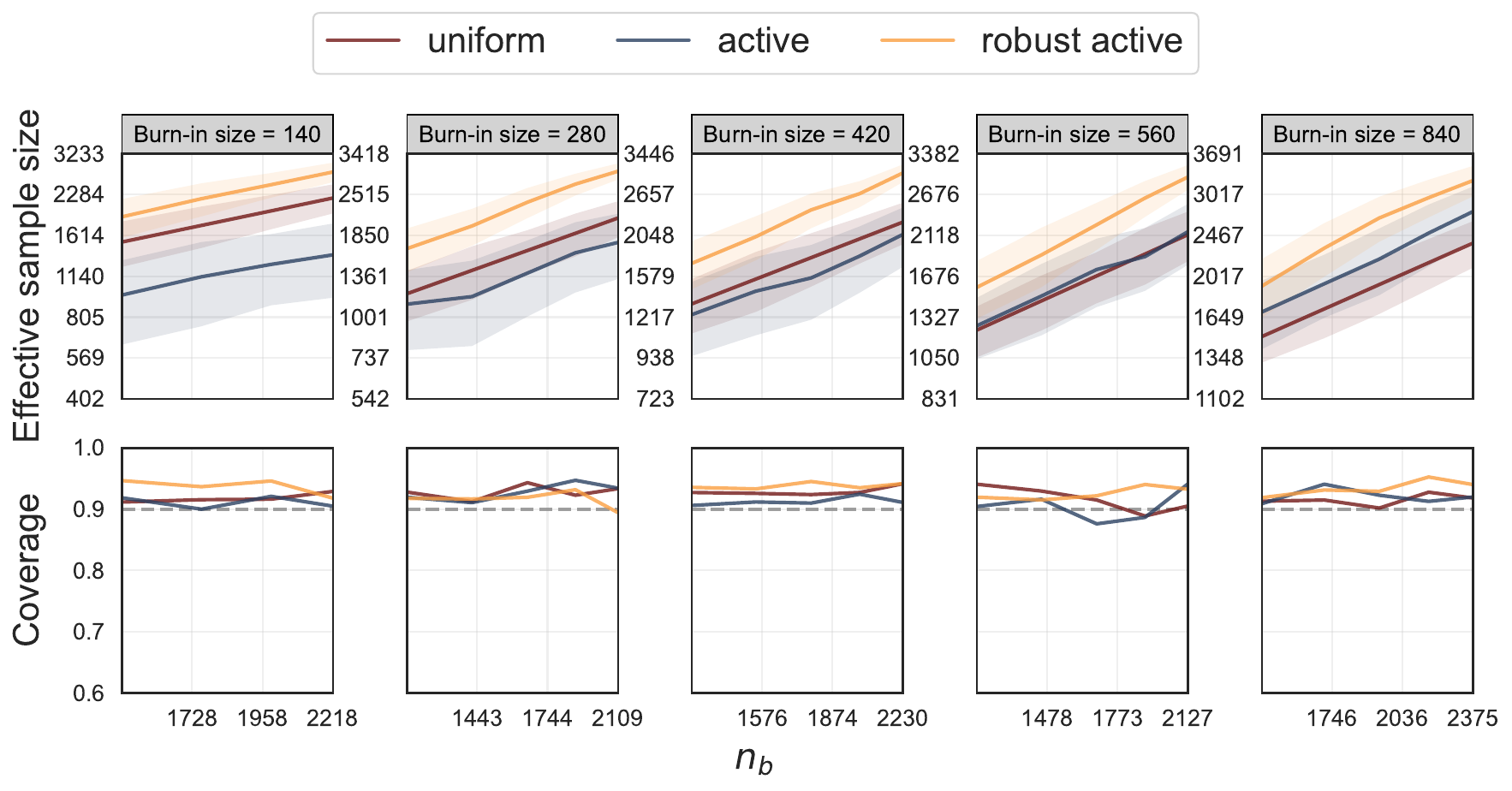}
    \caption{\textbf{Effective sample size and coverage on US Census data}, for varying burn-in dataset sizes. We compare uniform, active, and robust active sampling, for different values of the sampling budget $n_b$. The target of inference is the relationship between age and income, estimated via a linear regression. We show the mean and one standard deviation of the effective sample size estimated over $500$ trials; in each trial we independently sample the observed labels.}
    \label{fig:census_robust_coverage}
\end{figure}

\begin{figure}[H]
    \centering
    \includegraphics[width=0.9\linewidth]{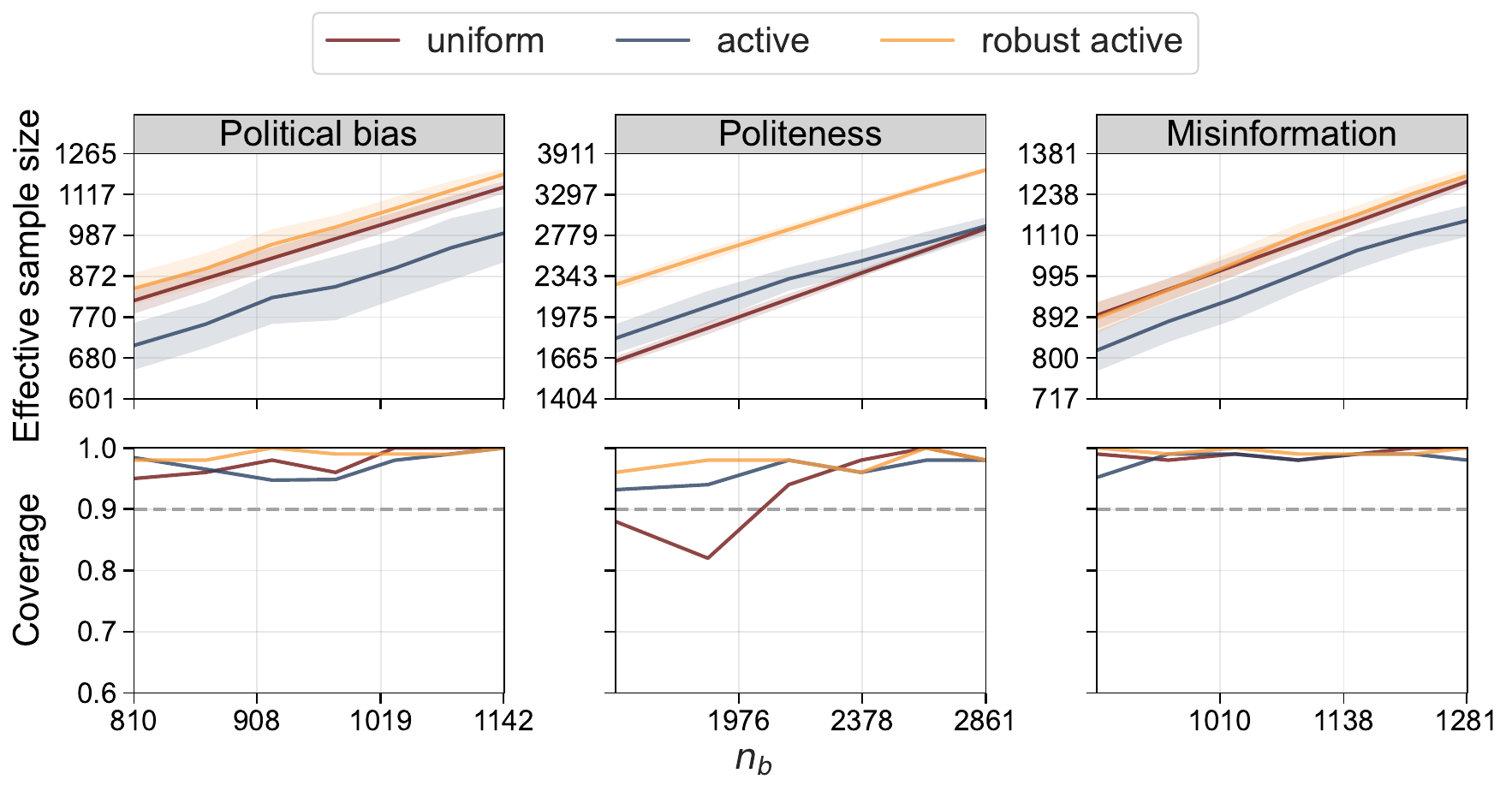}
    \caption{\textbf{Effective sample size and coverage on social science text annotation datasets}. We compare uniform, active, and robust active sampling, for different values of the sampling budget $n_b$. The targets of inference are (left to right) the prevalence of right-leaning political bias, the relationship between hedging and politeness, and the prevalence of misinformation. We show the mean and one standard deviation of the effective sample size estimated over $500$ trials; in each trial we independently sample the observed labels.}
    \label{fig:llm_ess_coverage}
\end{figure}

\begin{figure}[H]
    \centering
    \includegraphics[width=1\linewidth]{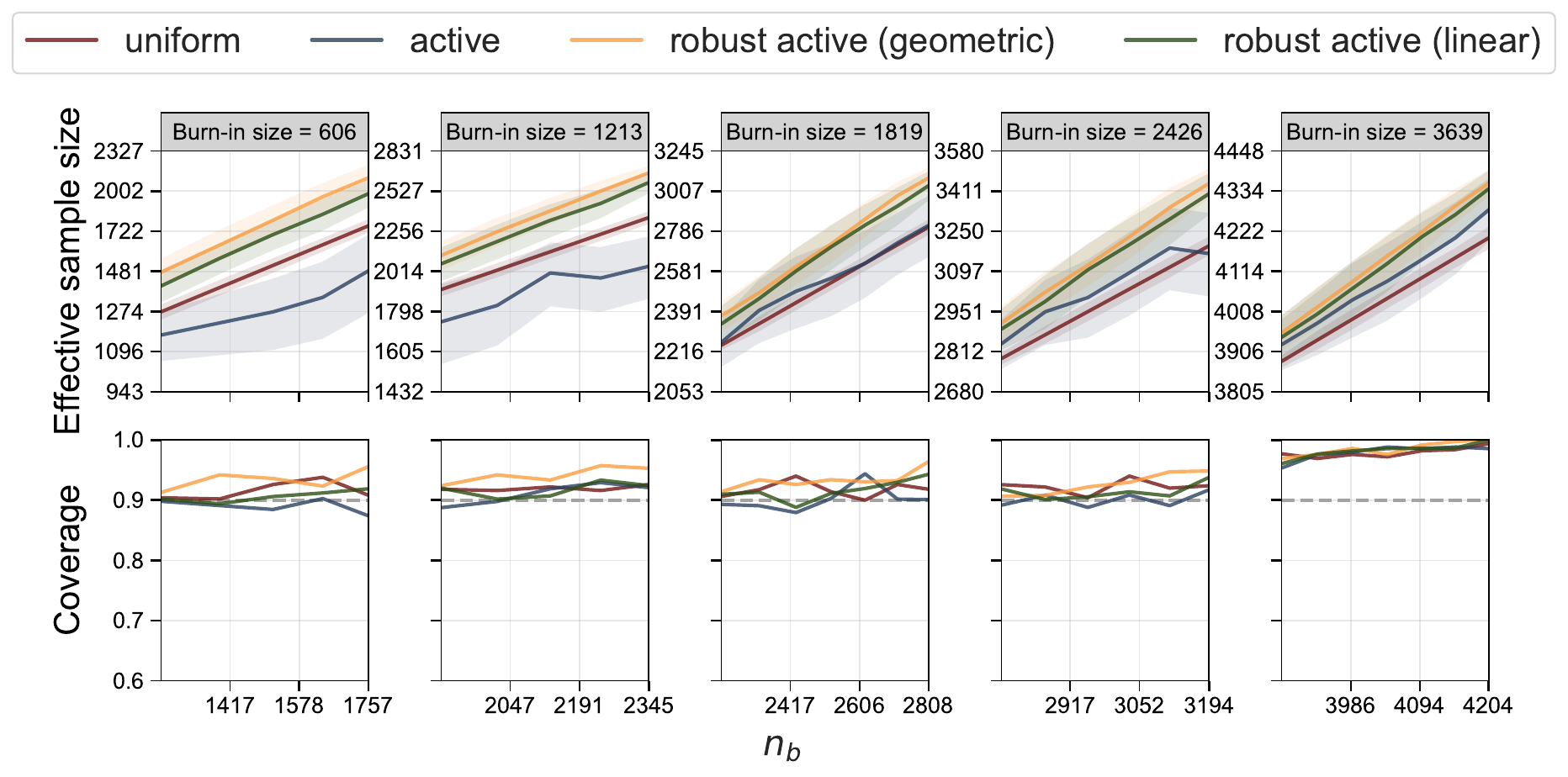}
    \caption{\textbf{Effective sample size (top) and coverage (bottom) on Pew post-election survey data}, for varying burn-in dataset sizes with respect to different proportions of the data. We compare uniform, active, and robust active sampling with geometric and linear paths, for different values of the sampling budget $n_b$. The target of inference is the approval rate of a presidential candidate. We show the mean and one standard deviation of the effective sample size estimated over $500$ trials; in each trial we independently sample the observed labels.}
    \label{fig:13}
\end{figure}

\subsection{Burn-in size v.s. robustness constraint $\mathcal{C}$}

In addition to optimized $\rho_{\mathrm{robust}}$ along the path, we also provided the optimized value $c$ in the robustness constraint $\mathcal C = \{\boldsymbol{\epsilon} : \|\boldsymbol{\epsilon}\|_2 \leq c\}$. As expected, we observe a more conservative constraint when the errors are poorly estimated.

\begin{figure}[H]
    \centering
    \includegraphics[width=0.5\linewidth]{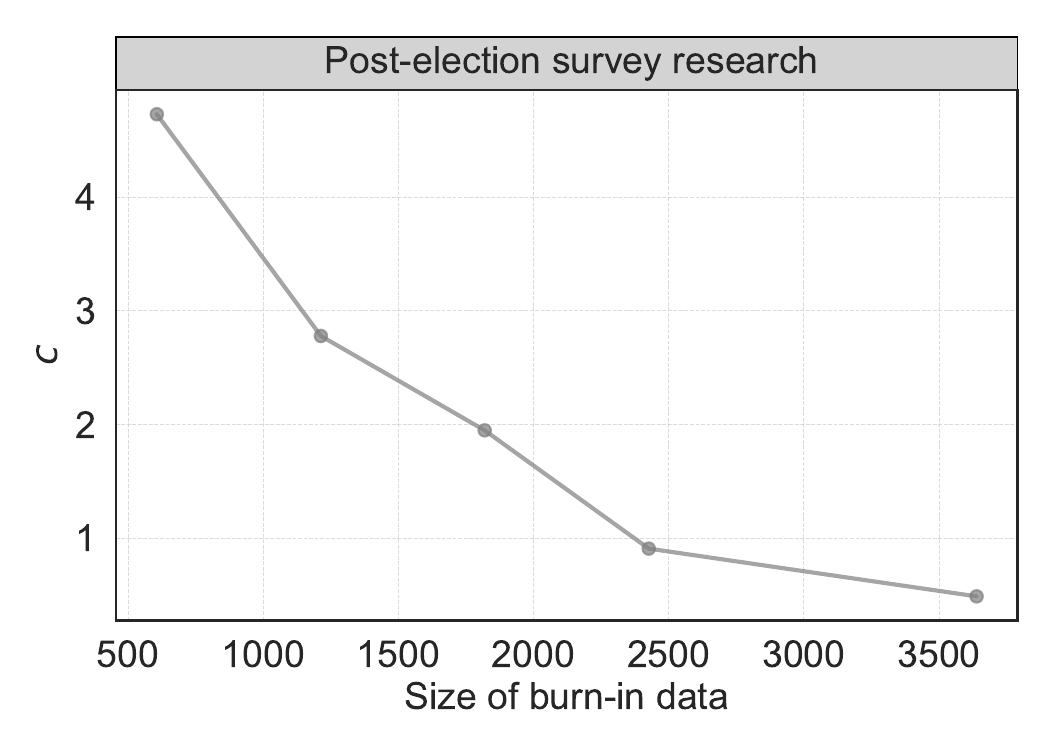}
    \caption{\textbf{Optimized value $c$ along the geometric path} as a function of the size of the burn-in data for the post-election survey data. }
    \label{fig:14}
\end{figure}

\subsection{Sensitivity to step size}
When we solve the optimization problem \eqref{eq:robust_rho}, we employ a grid search for $\rho$ in the outer loop. We conducted experiments to explore different step sizes of the grid search and confirmed the robustness of our results to step-size selection, as shown below. The gap in effective sample size between these two estimators is minimal, and both significantly outperform uniform and active baselines.
\begin{figure}[H]
    \centering
    \includegraphics[width=1.0\linewidth]{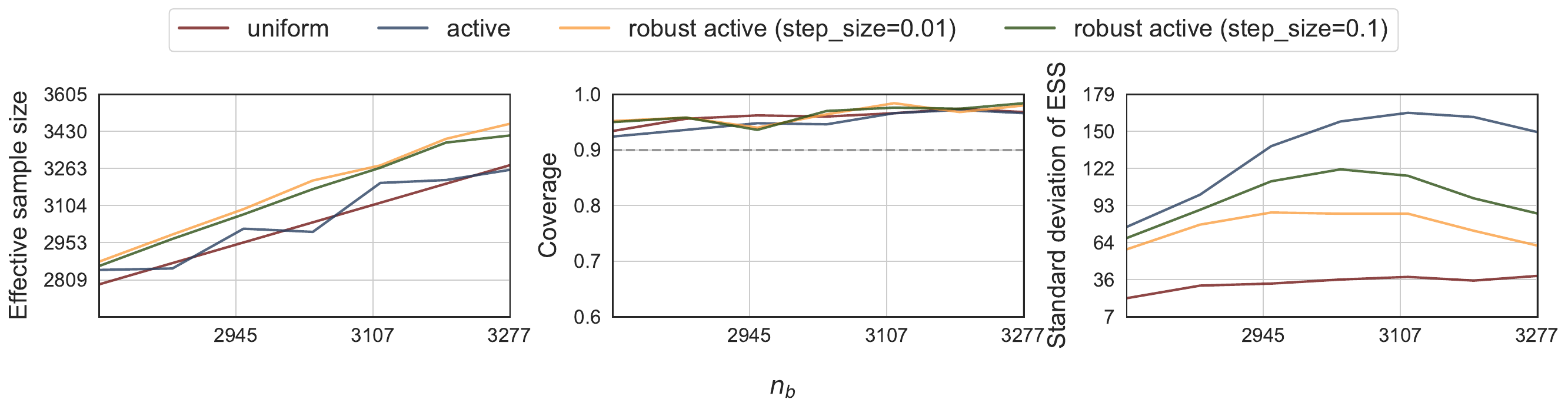}
    \caption{\textbf{Effective sample size on Pew post-election survey data}, for different step sizes in grid search for $\rho$. We compare uniform, active, and robust active sampling with grid search step sizes of 0.01 and 0.1. The target of inference is the approval rate of a presidential candidate. We show the mean and one standard deviation of the effective sample size estimated over $500$ trials; in each trial we independently sample the observed labels.}
    \label{fig:15}
\end{figure}

\end{document}